\documentclass[wcp]{jmlr}

\usepackage{longtable}

\usepackage{booktabs}

\def\dist{\text{dist}}
\def\vec{\text{vec}}
\def\tr{\text{tr}}
\def\reals{\mathcal{R}}
\def\E{\mathbb{E}}

\def\dist{\textnormal{dist}}

\usepackage{color}
\usepackage{algorithmic}
\usepackage{amssymb}
\usepackage{wrapfig}
\usepackage{comment}
\usepackage{bbm}
\usepackage{url}
\usepackage{xr}

\usepackage{listings}
\usepackage{graphicx}
\usepackage{hyperref}
\usepackage{url}

\newcommand{\nw}[2]{w^{(#1)}_{\# #2}}

\newcommand{\tparw}[3]{w^{(#1), \parallel }_{\# #2, #3}}
\newcommand{\tpenw}[3]{w^{(#1), \perp }_{\# #2, #3}}
\newcommand{\tw}[3]{w^{(#1)}_{\# #2, #3}}
\newcommand{\w}[1]{w^{(#1)}}
\newcommand{\parw}[1]{w^{(#1),\parallel}}
\newcommand{\penw}[1]{w^{(#1),\perp}}

\makeatletter
\let\Ginclude@graphics\@org@Ginclude@graphics 
\makeatother

\jmlrvolume{157}
\jmlryear{2021}
\jmlrworkshop{ACML 2021}

\title[Understanding How Over-Parametrization Leads to Acceleration]{Understanding How Over-Parametrization Leads to Acceleration: A case of learning a single teacher neuron}

  \author{\Name{Jun-Kun Wang} \Email{jun-kun.wang@yale.edu}\\
  \addr Yale University 
  \AND
  \Name{Jacob Abernethy} \Email{prof@gatech.edu}\\
  \addr Georgia Institute of Technology
 }

\editors{Vineeth N Balasubramanian and Ivor Tsang}

\begin{document}

\maketitle

\begin{abstract}
Over-parametrization has become a popular technique in deep learning. 
It is observed that by over-parametrization, 
a larger neural network needs a fewer training iterations than a smaller one to achieve a certain level of performance --- namely, over-parametrization leads to acceleration in optimization.
However, despite that over-parametrization is widely used nowadays, little theory is available to explain the acceleration due to over-parametrization.
In this paper, we propose understanding it by studying a simple problem first. Specifically, we consider the setting that there is a single teacher neuron with quadratic activation, where over-parametrization is realized by having multiple student neurons learn the data generated from the teacher neuron. 
We provably show that over-parametrization
helps the iterate generated by gradient descent to enter the neighborhood of a global optimal solution that achieves zero testing error faster.
\end{abstract}


\section{Introduction}

Over-parametrization has become a popular technique in deep learning, as it is now widely observed larger neural nets can achieve better performance. Furthermore, a larger network can be trained to achieve a certain level of prediction performance with fewer iterations than that of a smaller net.
This observation, to our knowledge, can be dated back as early as the work of \citet{LSS14}, who try different levels of over-parametrization and report that SGD converges much faster and finds a better solution when it is used to train a larger network. However, the reason why over-parametrization can lead to an acceleration still remains a mystery, and very little theory has helped explain the observation, with perhaps the notable exception of \citep{ACH18}.
\citet{ACH18} consider over-parametrizing a single-output linear regression with $l_p$ loss for $p>2$--the square loss corresponds to $p=2$--and they study the linear regression problem by replacing the model $w \in \reals^d$ by another model $w_1 \in \reals^d$ times a scalar $w_2 \in \reals$. They show that the dynamics of gradient descent on the new over-parametrized model are equivalent to the dynamics of gradient descent on the original objective function with an adaptive learning rate plus some momentum terms. However, in practice, people actually use the techniques of over-parametrization, adaptive learning rate, and momentum simultaneously in deep learning (see e.g. \citep{HHS17,KB15}), as each technique appears to contribute to performance and they may, to some extent, be complementary. It has been suggested that over-parameterizing a model leads implicitly to an adaptive learning rate or momentum, but this does not appear to fully explain the performance improvement.

\begin{figure}[t]
\centering
\subfigure[\footnotesize Obj. (\ref{obj:over}) on training data]{
\includegraphics[width=0.3\textwidth]{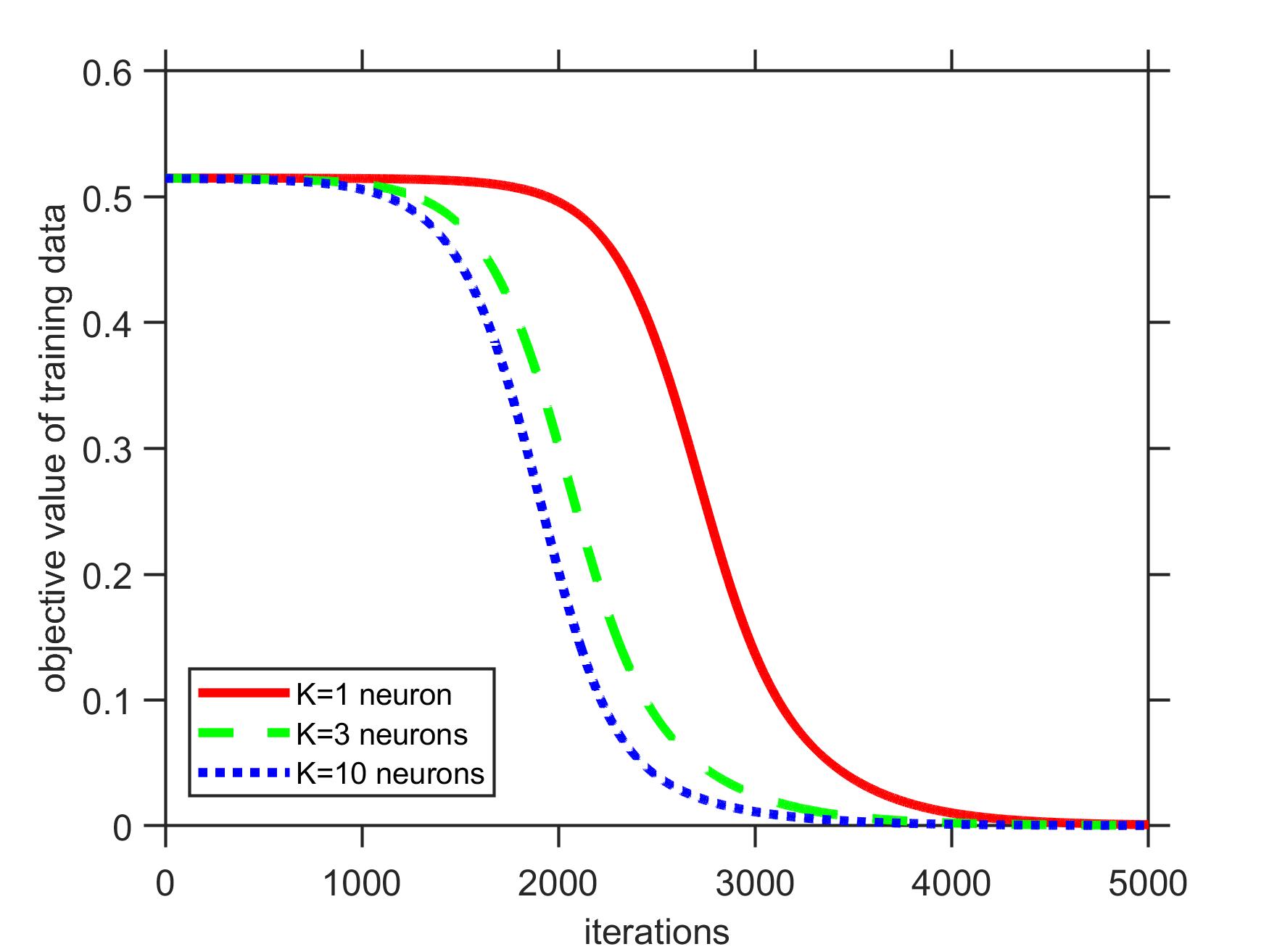}
}
\subfigure[\footnotesize  $\dist(W^{\#K}_t, w_*)$ vs. $t$]{
\includegraphics[width=0.3\textwidth]{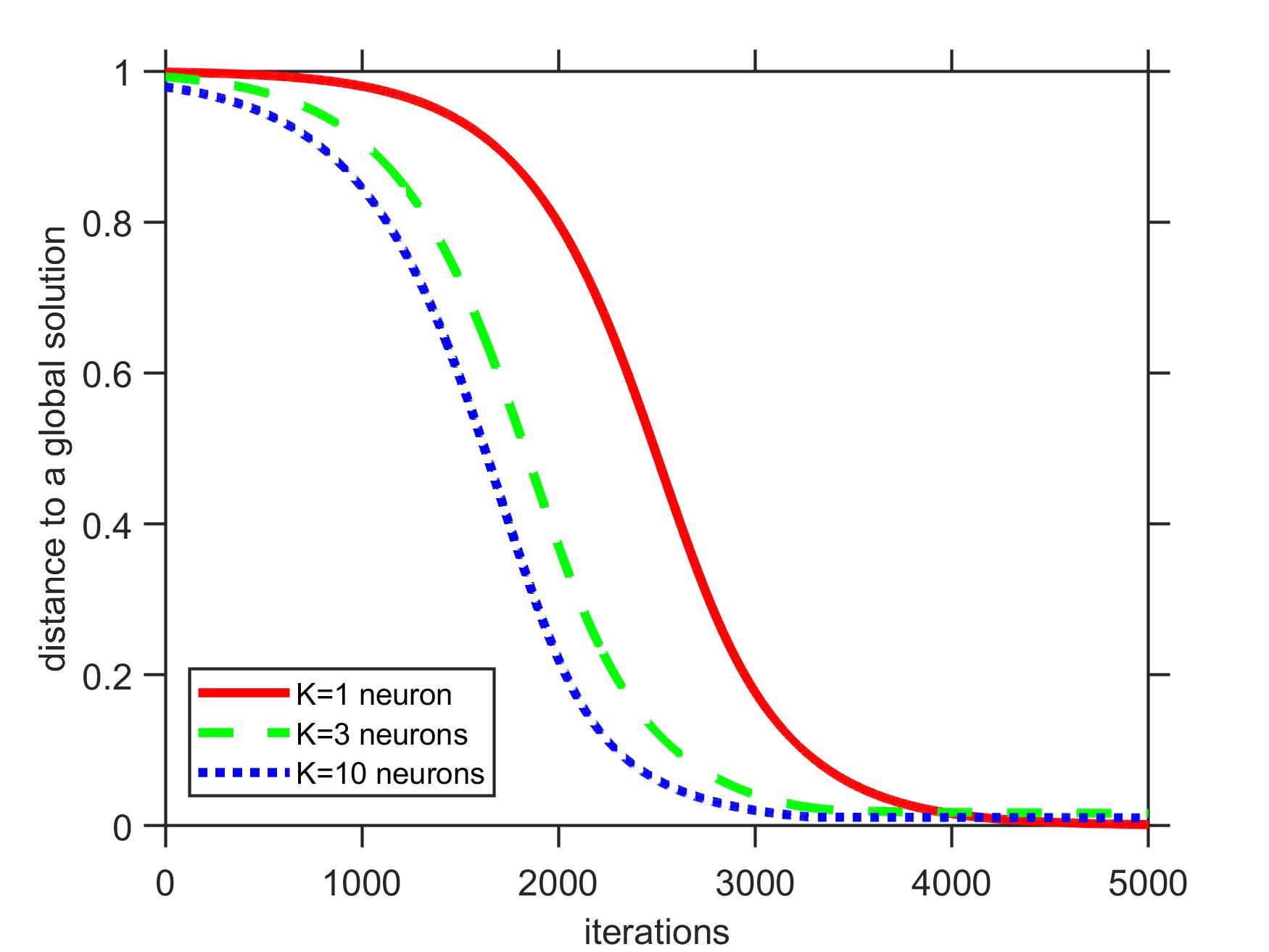}
}
\subfigure[\footnotesize  Quantity (\ref{eq:vHv}) over $t$]{
\includegraphics[width=0.3\textwidth]{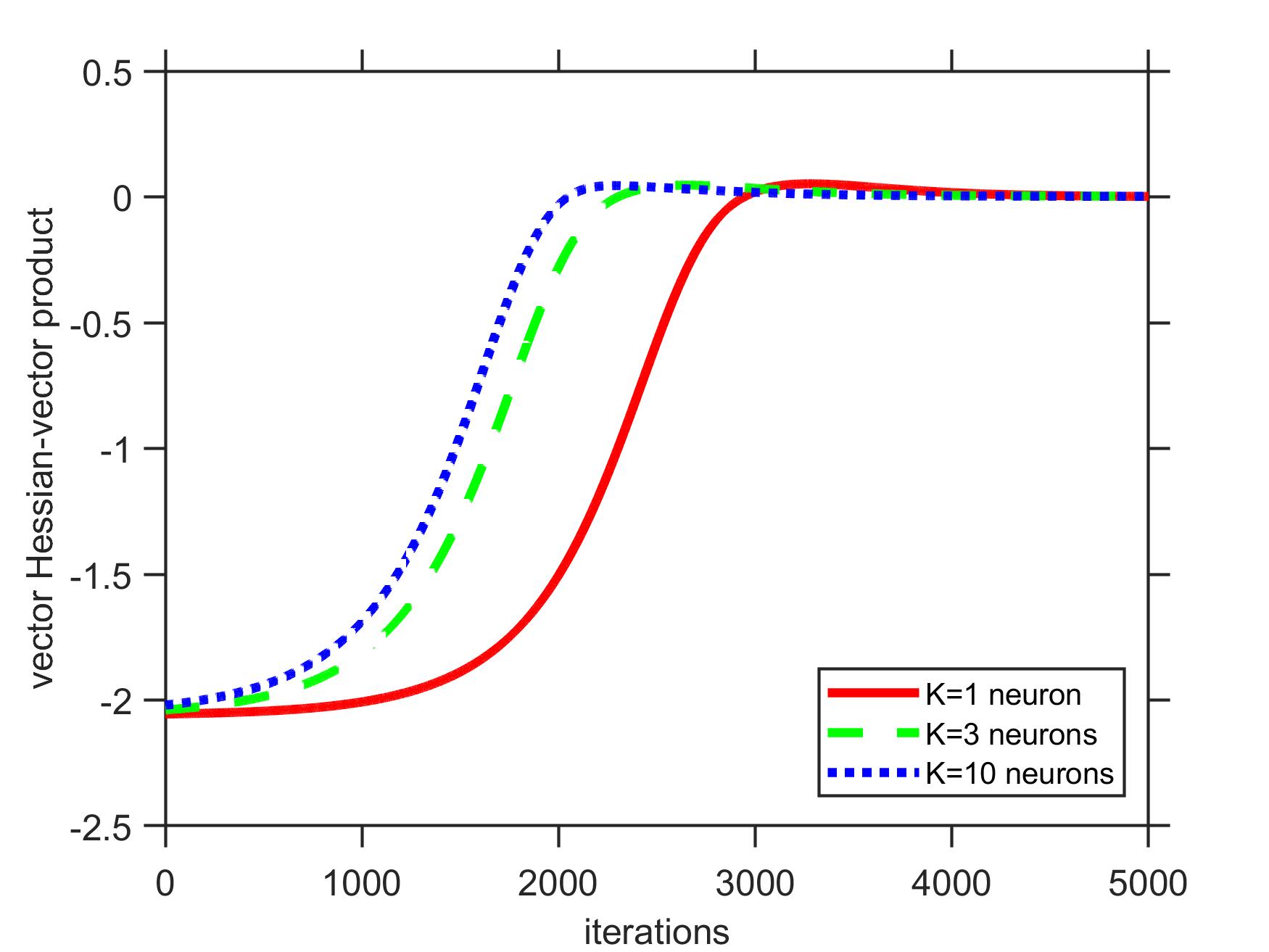}
}
     \caption{\footnotesize
In the experiment, we set the dimension $d=10$ and the number of training samples $n=200$.
Additional $200$ samples are sampled and served as testing data.
We let $w_* = e_1$ with $e_1$ being the unit vector.
Each neuron $w^{(k)} \in \reals^d$ ($k\in [K]$) of the student network is initialized by sampling from an isotropic distribution 
and is close to the origin (i.e. $w_0^{(k)} \sim 0.01 \cdot N(0,I_d/d)$ ).
We apply gradient descent with the same step size $\eta = 0.001$ to train different sizes of neural networks. Each curve represents the progress of gradient descent for different $K$. 
     Subfigure (a): Objective value (\ref{obj:over}) vs. iteration $t$ on training data (testing data, respectively). 
For subfigure (b) and (c), please see Section~3 for the precise definition and details. 
     }
     \label{fig:trts}
\end{figure}

To understand the benefits of overparameterization, let us begin by studying a simple canonical problem: a single teacher neuron $w_* \in \reals^d$ with 
quadratic activation function. Specifically, the label $y_i$ and the design vector $x_i \sim \mathcal{N}( 0, \mathcal{I}_d )$ of sample $i$
satisfies $y_i = (x_i^\top w_*)^2$.
Therefore, the \emph{standard} objective function for learning the teacher neuron $w_*$ is 
\begin{equation} \label{obj}
\begin{split}
\textstyle
 \min_{w \in \reals^d} f(w)  :=    \frac{1}{4n} \sum_{i=1}^n \big( (x_i^\top w)^2 - y_i    \big)^2,
\end{split}
\end{equation}
where $n$ denotes the number of samples.
Problem (\ref{obj}) is called \emph{phase retrieval} in signal processing literature, which has real applications in physical science such as astronomy and microscopy (see e.g. \citep{CSV13}, \citep{AS20}, \citep{SECCMS15}). There are also some specialized algorithms designed for achieving a better sample complexity or computational complexity to recover $w_*$ modulo the unrecoverable sign. 
We choose this problem as a starting point of understanding acceleration due to over-parametrization. Specifically, we consider the following way to over-parametrize the original objective (\ref{obj}),
\begin{equation} \label{obj:over}
\begin{split}
& \textstyle
 \min_{W \in \reals^{d \times K} }f(W)  :=   
\textstyle  \frac{1}{4n} \sum_{i=1}^n \big( (x_i^\top w^{(1)})^2 + (x_i^\top w^{(2)})^2 + ... + (x_i^\top w^{(K)})^2  - y_i    \big)^2,
\end{split}
\end{equation}
where $w^{(j)}$ denote the $j_{th}$ column of the weight matrix $W \in \reals^{d \times K}$.
Optimizing the objective function (\ref{obj:over}) can be viewed as training a student network with $K$ student neurons. 
While a $d$ dimensional model exists which perfectly predicts the labels generated by the teacher neuron (i.e. $\pm w_*$), one can consider training a much larger model instead (i.e. $d \times K$ number of parameters). Note that if $K=1$, objective (\ref{obj:over}) reduces to the original objective (\ref{obj}). On the other hand, a larger $k>1$ means a higher degree of over-parametrization.

Let us now establish, empirically, the clear advantage of over-parametrization for accelerating the learning and optimization process. 
We tried $K=\{1,3,10\}$ number of student neurons in (\ref{obj:over}),
which represent different degrees of over-parametrization,
and we applied gradient descent to train the networks. 
Figure~\ref{fig:trts} shows the results. 
The empirical findings displayed on the figure are quite stark.
First, we see that over-parametrization not only helps to decrease the training error faster but also decrease the testing error faster (c.f. subfigure (a) and (b)). Furthermore, the generalization error is very small since both training error and testing error approach zero. 
Second, regardless of the size $K$, a common pattern is that the dynamics of gradient descent can be divided into two stages.
In the first stage, gradient descent makes little progress on decreasing the function value; while in the second stage, the iterate generated by gradient descent exhibits a linear convergence to a global solution.
Specifically, by over-parametrization, gradient descent spends fewer iterations in the first stage and enters the linear convergence regime that makes the fast progress more quickly. 
We provide more results in Section~\ref{app:diff}. 
Specifically, we also tried different values of the step size $\eta$ and we observed similar patterns as Figure~\ref{fig:trts} shows.
Even when gradient descent uses the best step size for each model with different neurons $K$, 
we still observe that gradient descent enters the linear convergence regime faster for a larger model. Thus, the acceleration due to over-parametrization \emph{cannot} be simply explained by that gradient descent uses a larger \emph{effective} step size, as the effect due to parameters $\eta$ and $K$ is complementary in the experiment.

\section{Related works}
\noindent
\textbf{Over-parametrization:}
Though our work focuses on understanding why over-parametrization leads to acceleration in optimization (i.e. improving the convergence time), we also want to acknowledge some related works of understanding over-parametrization in different aspects (e.g. \citep{ACHL19,BG19,EGKZ20,GASKZ19,T20}.
There is also a trend of works studying how over-parametrization changes 
the optimization landscape of empirical risk minimization 
for neural nets with quadratic activation. 
(e.g. \citep{DL18,GKZ19,GWZ19,KLD19,SJL18,NH17,VBB19,MVZ20}). 
The goals of these works are different from ours.

\noindent
\textbf{Quadratic activation and matrix sensing:}
The optimization landscape of problem (\ref{obj}) (i.e. phase retrieval) and its variants has been studied by \citep{DDP18,S14,SQW16,WSW16}, which shows that as long as the number of samples is sufficiently large, it has no spurious local optima and all the local optima are globally optimal. 
\citet{CCFMY18} provably show that applying gradient descent with an isotropic random initialization for solving (\ref{obj}) leads to an optimal solution that recovers the teacher neuron $w_*$ modulo the unrecoverable sign. In this work we show that an over-parametrized student network 
trained by gradient descent takes even fewer iterations to recover $w_*$.
We also note that the optimization problem (\ref{obj:over}) can be rewritten as the form of matrix sensing (see e.g. \cite{GWBNS17,LMCC19,LMZ18,GBL19}).

\noindent
\textbf{Learning a single neuron:}
Studying learning a single neuron in non-convex optimization is not new (e.g. \cite{GKM19,YS20,KKSK11,GKKT17,KSA19,S17,MBM17,FCG20}).
However, those works are not for showing faster convergence by over-parametrization. The goals are different.

\section{Preliminaries}

\noindent 
\textbf{Notations and assumptions:}
We use the notation $W^{\#K}:= [ \nw{1}{K}, \dots, \nw{K}{K} ]  \in \reals^{d \times K}$ to represent the weights of a student network with $K$ number of neurons. Each column $k$ of the matrix $W^{\#K}$ is the weight vector $\nw{k}{K} \in \reals^d$ that corresponds to the neuron $k$ of the student network. Thus, $W^{\#1}  :=[ \nw{1}{1} ] \in \reals^d$ is the network consists of a single neuron; while for $K>1$, the notation represents the weights of an over-parametrized network. 
In this paper, without loss of generality, we assume that $w_* = \| w_* \| e_1 \in \reals^d$ with $e_1$ being the unit vector whose first element is $1$.
We define the parallel component $\tparw{k}{K}{t}$ and the perpendicular component $\tpenw{k}{K}{t}$ in iteration $t$ as follows,
\begin{equation}
\begin{aligned}
\textstyle
\tparw{k}{K}{t} & \textstyle := \langle \tw{k}{K}{t}, w_* \rangle = \| w_* \| 
\tw{k}{K}{t}[1] 
\\
\textstyle \tpenw{k}{K}{t} & \textstyle := [ \tw{k}{K}{t}[2], \dots, \tw{k}{K}{t}[d] ]^\top.
\end{aligned}
\end{equation}
Namely, the parallel component $\tparw{k}{K}{t}$ is the projection of a student neuron $k$ learned in iteration $t$ on the teacher neuron $w_*$; while the perpendicular component $\tpenw{k}{K}{t}$ is the $d$-$1$ dimensional sub-vector of $\tw{k}{K}{t}$ excluding the first element.
We assume that each neuron $k$ of each network, $\tw{k}{K}{0} \in \reals^d$, is initialized i.i.d. randomly from an isotropic distribution (e.g. gaussian distribution).

\noindent
\textbf{Metric of the progress in optimization:}
The first challenge to show that over-parametrization leads to acceleration in optimization is the design of the metric for the comparison. Since different $K$'s corresponds to different optimization problems, 
the notion of \emph{acceleration} here is non-standard in optimization literature. In the optimization literature, acceleration usually means that an algorithm takes a fewer iterations than other algorithms for the \emph{same} optimization problem. Fortunately, by exploiting the problem structure, we can have a natural metric of the progress as follows.
For any size of the student network $W \in \reals^{d \times K}$, we consider
\begin{equation} \label{dis}
\textstyle \dist( W, w_*):= \min_{q \in \reals^K: \| q\|_2 \leq 1}  \|  W - w_*q^\top \|.
\end{equation}
This is due to the observation that for any $K$,
the global optimal solutions of (\ref{obj:over}) that achieve zero testing error are
$w_* q^\top \in \reals^{d \times K}$ for any $q \in \reals^K$ such that $\| q \|_2 = 1$.
To see this, substitute $W = w_* q^\top  \in \reals^{d \times K}$ into (\ref{obj:over}).
We have that for any $x_i \in \reals^d$ it holds that
$
(x_i^\top w^{(1)})^2 + (x_i^\top w^{(2)})^2 + ... + (x_i^\top w^{(K)})^2  - y_i =  \| x_i^\top W \|^2_F - (x_i^\top w_*)^2  = \tr\big( (x_i^\top w_* q^\top)^\top (x_i^\top w_* q^\top) \big) - (x_i^\top w_*)^2   = 0 $.
Therefore, the metric $\dist( W, w_*)$ as be viewed as a surrogate of the testing error.
In particular,
$\dist( W_t, w_*)$ represents the distance of the current iterate $W_t$ and its closest global optimal solution to the over-parametrized objective (\ref{obj:over}) that achieves zero testing error. 
Note that the argmin of (\ref{dis}) is
$q_* := \frac{ W^\top w_*}{ \|W^\top w_*\|_2 } = \arg\min_{q \in \reals^K: \| q\|_2 \leq 1}  \|  W - w_*q^\top \|.$
On sub-figure (c) of Figure~\ref{fig:trts}, we plot the distance of the iterates generated by gradient descent and its closet global optimal solution for different sizes $K$ of neural nets. We see that over-parametrization enables shrinking the distance $\dist(W^{\#K}_t, w_*)$ faster. 

\noindent
\textbf{Gradient descent dynamics:}
For the notation brevity, we will suppress the symbol $\#K$ when it is 
clear in the context.
The gradient of a student neuron $w^{(k)}$ for the over-parametrized problem (\ref{obj:over}) is 
\begin{equation}
\begin{aligned}
\textstyle
\nabla_{w^{(k)}} f(W) & \textstyle  := \frac{1}{n} \sum_{i=1}^n \big( (x_i^\top w^{(1)})^2 + (x_i^\top w^{(2)})^2 + \dots + (x_i^\top w^{(K)})^2  - y_i   \big) (x_i^\top w^{(k)}) x_i.
\end{aligned}
\end{equation}
Its expectation, which is 
the population gradient of a student neuron $w^{(k)}$ (i.e. gradient when the number of samples $n$ is infinite), is
\begin{equation} \label{Egrad}
\begin{split}
\textstyle
\underset{x \sim N(0,I_d)}{\E}[ \nabla_{w^{(k)}} f(W) ] 
= &
\textstyle \big( 3 \| w^{(k)} \|^2 - \| w_* \|^2 \big) w^{(k)}  - 2 ( w_*^\top w^{(k)}) w_*
\\ &
\textstyle
 + \sum_{j \ne k}^K  2 ((w^{(j)})^\top w^{(k)} ) w^{(j)} +  \| w^{(j)} \|^2 w^{(k)},
\end{split}
\end{equation}
where we use the fact that for any vector $u,v \in \reals^d $, 
$\textstyle 
\E_{x \sim N(0,I_d)}[ (x^\top u)^3 x ]  = 3 \| u\|^2 u$ and 
$\textstyle
 \E_{x \sim N(0,I_d)}[ (x^\top u)^2 (x^\top v)  x]  = 2 ( u^\top v ) u + \| u \|^2 v$.
For $K=1$, (\ref{Egrad}) becomes
$\textstyle \E[ \nabla_{w^{(1)}} f(W^{\#1}) ]
= \big( 3 \| w^{(1)} \|^2 - \| w_* \|^2 \big) w^{(1)} - 2 ( w_*^\top w^{(1)}) w_* .$
If gradient descent uses the population gradient for the update (i.e. $\tw{1}{1}{t+1} = \tw{1}{1}{t} - \eta \E[ \nabla_{w^{(1)}} f(W^{\#1}_t) ] $),
then the dynamics of the student network consists of a single neuron (i.e. $K=1$) evolves as follows,
\begin{equation} \label{dyn:single}
\begin{aligned}
\textstyle \parw{1}_{t+1}
& \textstyle  = 
\parw{1}_{t} \big( 1 + \eta ( 3 \| w_* \|^2 - 3 \| \w{1}_{t} \|^2   ) \big) 
\\ 
\textstyle  \penw{1}_{t+1}
& \textstyle  = 
\penw{1}_{t} \big( 1 + \eta ( \| w_* \|^2 - 3 \| \w{1}_{t} \|^2   ) \big). 
\end{aligned}
\end{equation}
On the other hand, if a student network has $K>1$ neurons, the dynamics of each neuron $k$ of the student network evolves as follows,
\begin{equation} \label{eq:multi-dyn}
\begin{aligned}
\textstyle  \parw{k}_{t+1}
& \textstyle  =  
\underbrace{ \parw{k}_t  \big( 1 + 3\eta (  \| w_* \|^2 - \| \w{k}_t \|^2  ) \big) }_{\text{component A}} -  \big( \underbrace{ 2 \eta \sum_{j \neq k }^K   \langle \w{j}_t, \w{k}_t \rangle  \parw{j}_t
+ \eta \parw{k}_t \sum_{j \neq k}^K \| \w{j}_t \|^2
 }_{ \text{component B} }
\big)
\\ \textstyle 
\penw{k}_{t+1}
& \textstyle  =  
\underbrace{ \penw{k}_t  \big( 1 + \eta ( \| w_* \|^2  -3 \| \w{k}_t \|^2  ) \big) }_{\text{component C}} - \big( \underbrace{  2 \eta \sum_{j \neq k }^K  \langle \w{j}_t, \w{k}_t \rangle  \penw{j}_t
+ \eta \penw{k}_t \sum_{j \neq k}^K \| \w{j}_t \|^2
 }_{ \text{component D} } \big)
\end{aligned}
\end{equation}
where both the component $B$ of $\parw{k}_{t+1}$ and the component $D$ of $\penw{k}_{t+1}$ can be viewed as the terms due to the interaction of student neuron $k$ and the other student neurons.

\noindent
\textbf{More observations:}
On Subfigure (c) of Figure~\ref{fig:trts}, we plot a quantity over iterations, which is 
\begin{equation} \label{eq:vHv}
\textstyle
\vec( w_* q_t^\top - W^{\#K}_t)^\top \nabla^2 f(W^{\#K}_t) \vec(w_* q_t^\top - W^{\#K}_t),
\end{equation}
 where $\nabla^2 f(W^{\#K}_t) \in \reals^{dK\times dK }$ is the Hessian and $w_* q_t^\top$ is the closet global optimal solution to $W^{\#K}_t$ and
the notation $\vec(\cdot)$ represents the vectorization operation of its matrix argument.
The quantity can be viewed as a measure of the strong convexity. Specifically, if the quantity is larger than $0$, then it suggests that the current optimization landscape is strongly convex with respect to 
$w_* q_t^\top$.
Hence, the observation suggests that gradient descent enters a benign region faster for a larger student network.
In the following section, 
We will answer why over-parametrization helps gradient descent to enter a region that has the benign optimization landscape faster.

\section{Analysis}

In this section, we answer why over-parametrization leads to the acceleration.
We first show that gradient descent (GD) exhibits linear convergence to a global optimal solution when $\dist(W^{\#K},w_*)$ is small.
We then answer why over-parametrization helps the iterate to enter the 
neighborhood of a global optimal solution faster.
For the ease of analysis, we assume that gradient descent uses population gradient (\ref{Egrad}) for the update and we denote 
$\nabla F(W):= \E_{x}[ \nabla f(W) ] $ and
$\nabla^2 F(W):= \E_{x}[ \nabla^2 f(W) ] $ accordingly.

\subsection{When does the iterate generated by gradient descent enter a benign region?} 

We first introduce a key lemma. The lemma shows that whenever
the iterate is in the neighborhood of a global optimal solution, gradient descent has a linear convergence rate. 

\begin{lemma} \label{lem:linear}
(locally linear convergence)
Suppose that at time $t_0$, $\dist(W_{t_0}, w_*) :=\| W_{t_0} - w_* q_{t_0}^\top \| \leq \nu \| w_* \|$ where $W_{t_0} \in \reals^{d\times K}$, $\nu > 0$ satisfies $2 - 14 \nu - 2\nu^2 > 0$, and $q_{t_0} := \arg\min_{q \in \reals^K: \| q\|_2 \leq 1}  \|  W_{t_0} - w_*q^\top \|$. Then,
gradient descent with the step size $\eta \leq \frac{ 2 - 14 \nu - 2\nu^2 }{  (13 + 16 \nu^2 )^2 \| w_* \|^4}$
generates iterates $\{ W_{t} \}_{t \geq t_0}$ satisfying
$\textstyle
\dist^2(W_{t+1},w_*) \leq 
\big( 1 -  \eta  (2 - 14 \nu -2 \nu^2) \big) \dist^2( W_{t}, w_*).$
\end{lemma}
The proof is available in Appendix~\ref{app:benign}.
Note that the lemma holds for \emph{any} size $K$ of neural nets, as long as the condition, $\dist(W,w_*) \leq \nu \| w_* \|_2$ with the required $\nu$, is satisfied.
The condition ensures the locally linear convergence of gradient descent
and might be easily satisfied by having a larger number of student neurons $K>1$. Specifically, each neuron $\w{k} \in \reals^d$ of $W \in \reals^{d \times K}$ only needs to have a smaller component on the direction of $w_*$ in order to have $W$ be sufficiently close to the teacher neuron $w_*$ up to a transform $q^\top$, compared to the case when one only has a single student neuron ($K=1$). 
To support this argument, we plot the quantities
$|\tparw{k}{K}{t}| = |\langle \tw{k}{K}{t}, w_* \rangle|$
for each $k \in [K]$ of different student networks trained by gradient descent on Figure~\ref{fig:breakdown}. 
The figure shows that with more student neurons $K$, each $w^{(k)}$ only needs a smaller projection on $w_*$ for the aggregate projection $\sqrt{ \sum_{k=1}^K | \tparw{k}{K}{t} |^2 }$ to achieve certain level.
In the later subsections, we will provide a formal analysis.     

\begin{figure}[t]
\centering
\subfigure[\footnotesize $|\tparw{k}{K}{t}|$ ($K=3$)]{
\includegraphics[width=0.35\textwidth]{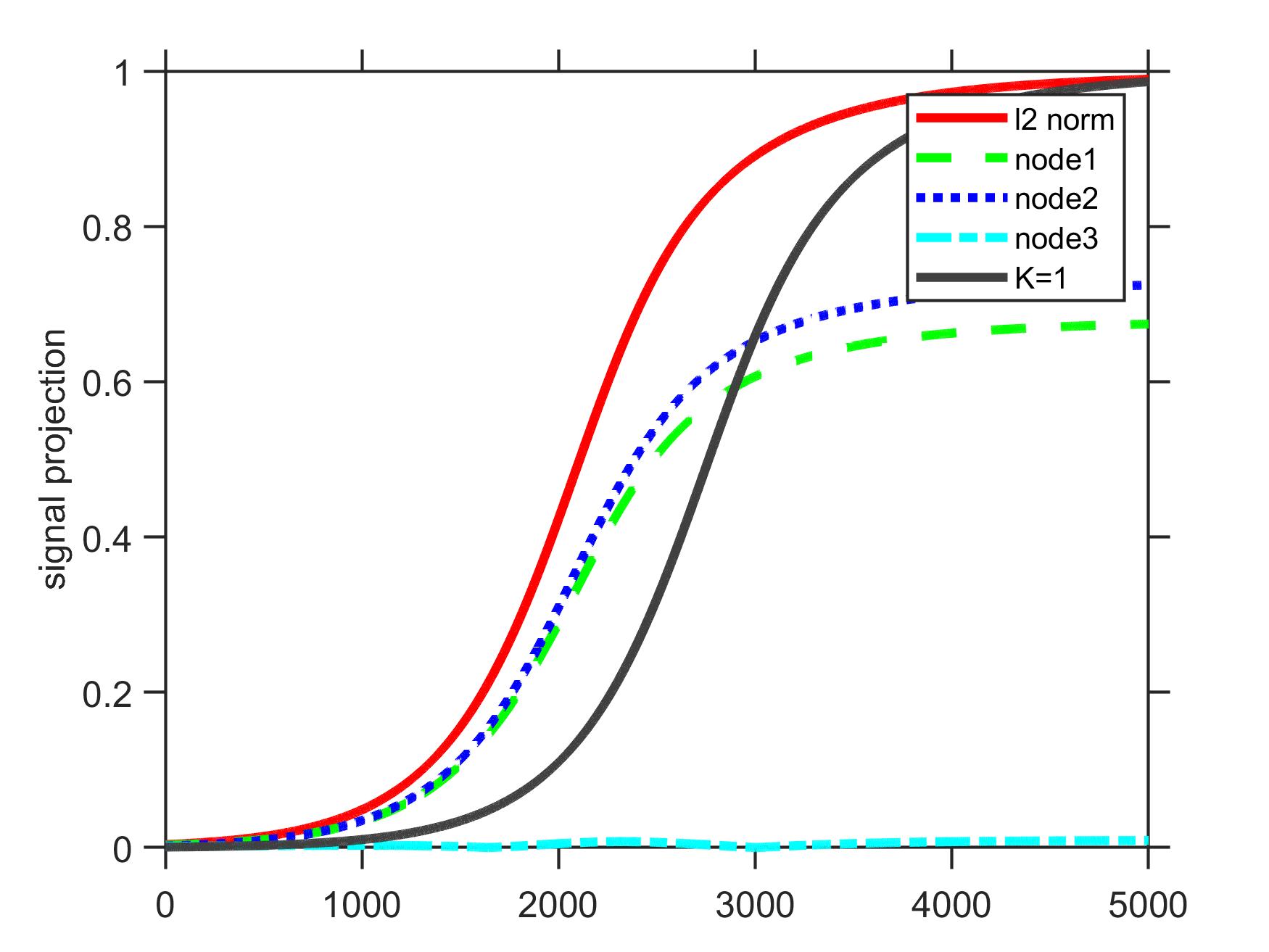}
}
\subfigure[\footnotesize $|\tparw{k}{K}{t}|$ ($K=10$)]{
\includegraphics[width=0.35\textwidth]{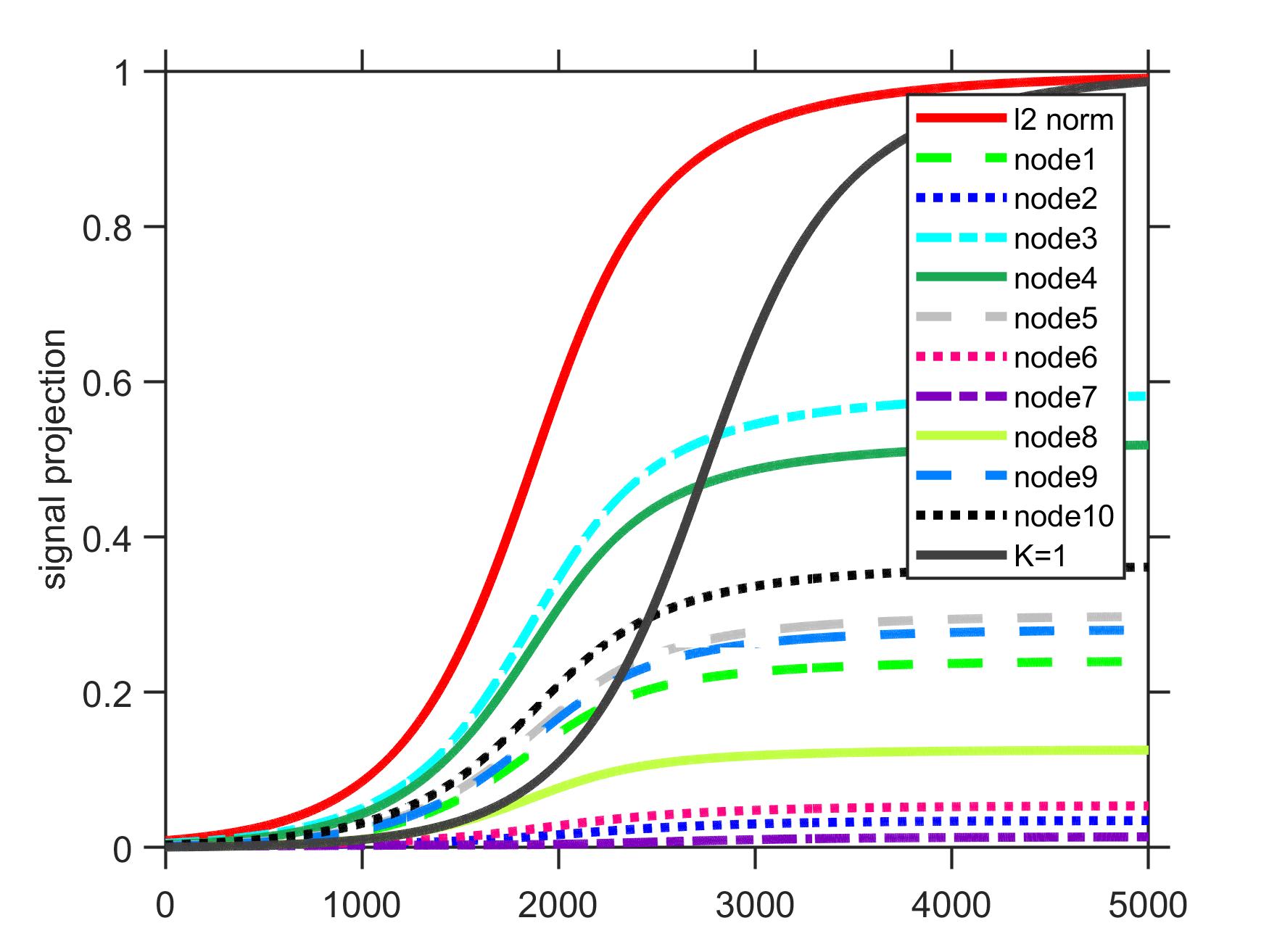}
}\\
\subfigure[\footnotesize $\| \tw{k}{K}{t} \|^2$  ($K=3$)]{
\includegraphics[width=0.35\textwidth]{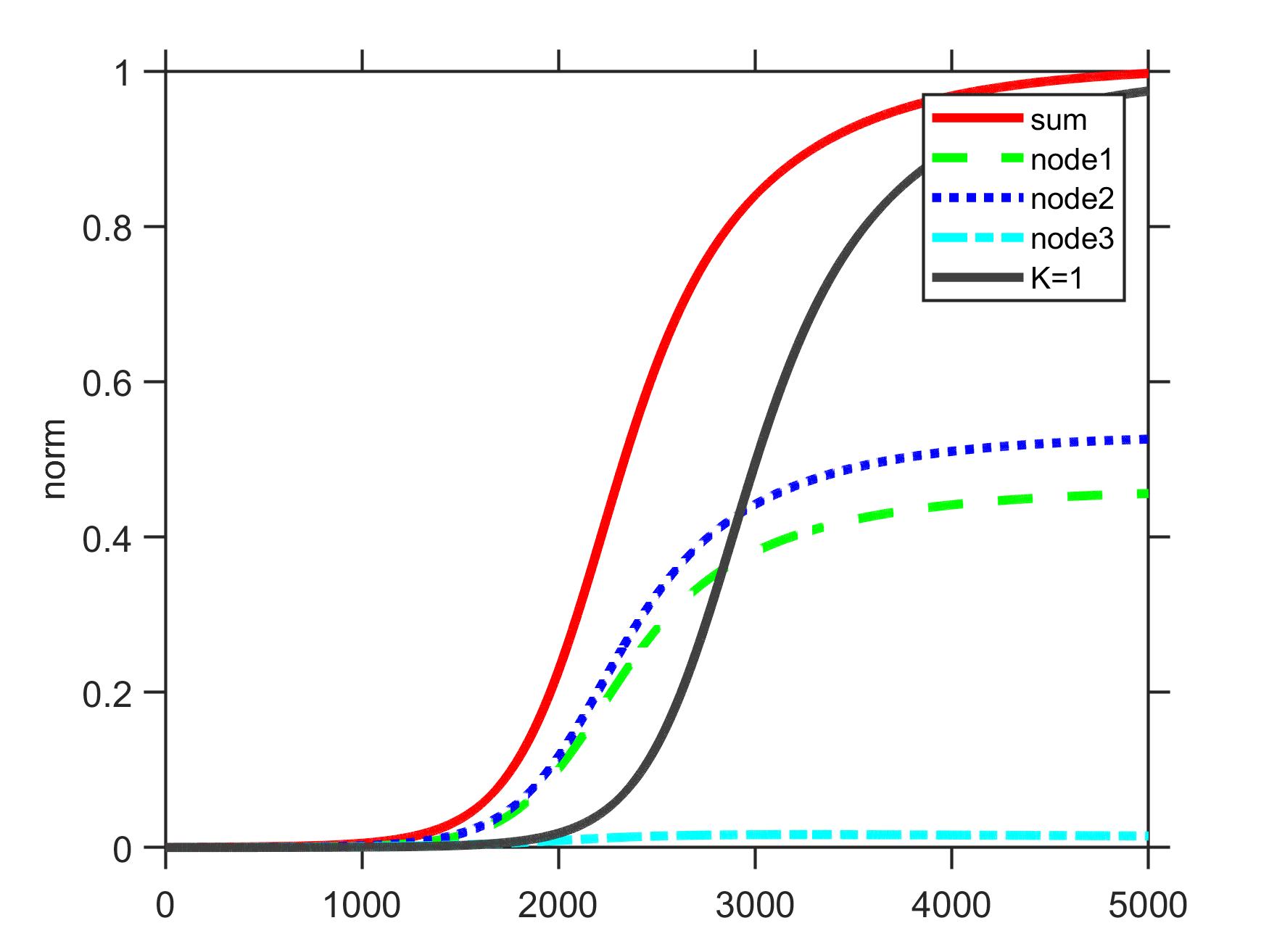}
}
\subfigure[\footnotesize $\| \tw{k}{K}{t} \|^2$  ($K=10$)]{
\includegraphics[width=0.35\textwidth]{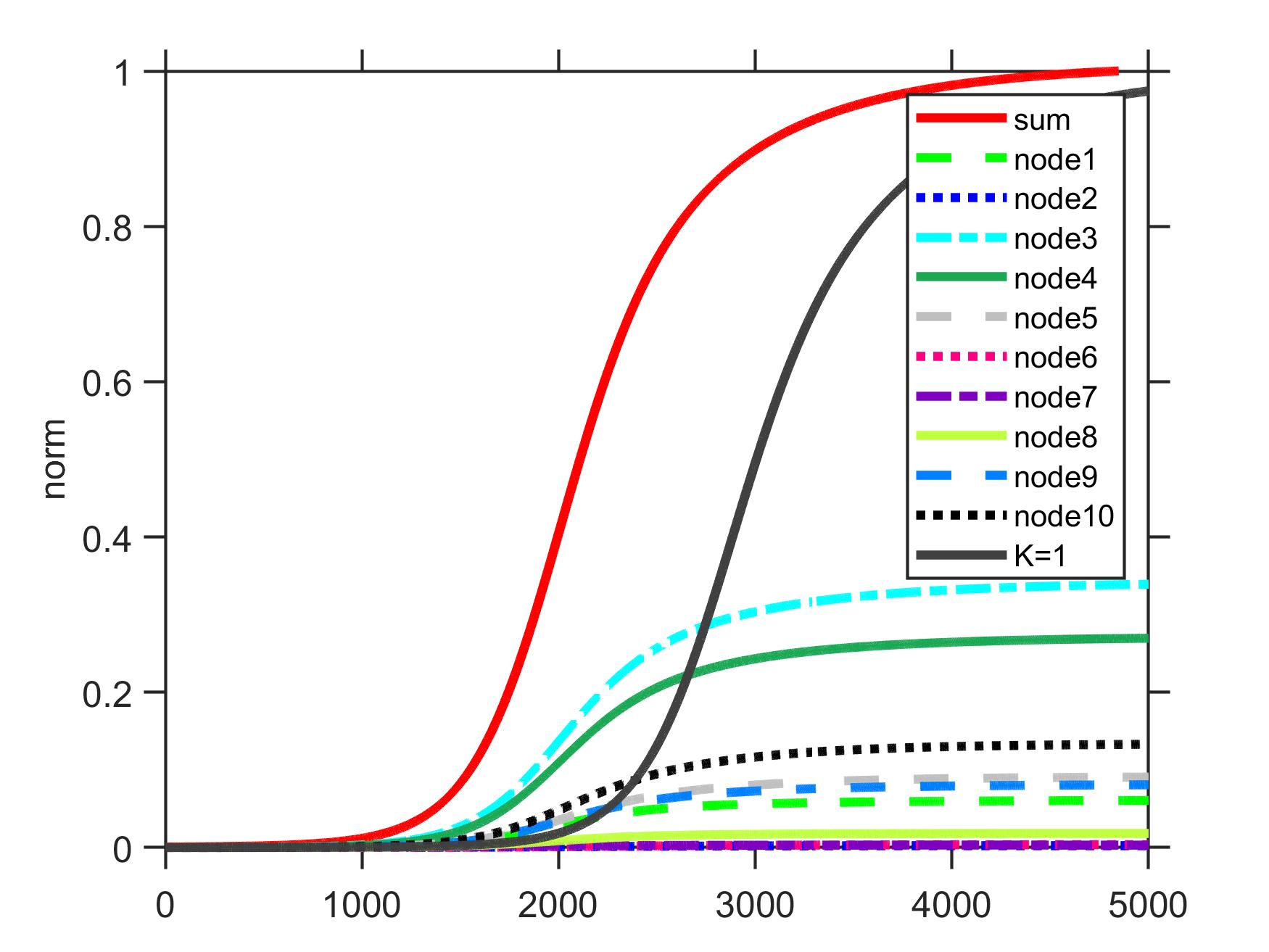}
}
     \caption{\footnotesize Subfigure (a) and (b): parallel component $|\tparw{k}{K}{t}|$ of each neuron $k$ vs. iteration $t$ for student networks with size $K=\{3,10\}$ trained by gradient descent.
     The curve  ``l2 norm'' represents $\sqrt{ \sum_{k=1}^K | \tparw{k}{K}{t} |^2 }$, which can be viewed as the aggregate projection of a student network $W \in \reals^{d \times K}$ on the teacher neuron $w_* \in \reals^d$. For comparison, we also plot the curve for $|\tparw{1}{1}{t}|$ labeled by $K=1$ on the same subfigures. Subfigure (c) and (d): the square norm $\| \tw{k}{K}{t} \|^2$ of each neuron $k$ for different $K$'s. The curve ``sum'' represents 
     $\| W^{\#K} \|^2_F$. For comparison, we also plot the curve for 
     $\|\tw{1}{1}{t}\|^2$ labeled by $K=1$ on the subfigures.} 
     \label{fig:breakdown}
\end{figure}
\vspace{-0.05in}

\subsection{How does gradient descent work for the student network consists of a single neuron?}
In this subsection, we analyze the case that the student network only has a single neuron. For brevity, we suppress the notation $\#1$ in the following.
Define
$\textstyle
T_{\gamma}:= \min \{t : 
| |w_t[1]| - \| w_* \| | \leq \gamma \text{ and } \| w_t^\perp \| \leq \gamma \}$.
 Note that if 
$| |w_t[1]| - \| w_* \| | \leq \gamma$ and that  $\| w_t^\perp \| \leq \gamma$, then $\dist^2( w_t, w_*) = |
|w_t[1]| - \|w_* \|  |^2 + \| w_t^\perp \|_2^2 \leq 2 \gamma^2$.
Let us assume that (C1) $\| w_* \| > 1.1 \gamma$ (strong signal) and  
(C2) $\gamma \geq 10 \| w_0 \|$ (small initialization).
Note that to invoke Lemma~\ref{lem:linear} for showing locally linear convergence after the iterate gets into a benign region, we will set $2 \gamma^2 = \nu^2 \|w_* \|^2$ with $\nu$ satisfying $2 - 14 \nu - 2\nu^2 > 0$ (i.e. $\nu \leq 0.141$); consequently, (C1) is trivially satisfied. 

\begin{theorem} \label{lem:single}
Suppose that the conditions (C1-C2) hold.
Assume that the step size satisfies $\eta \leq c / \| w_* \|^2$ for some sufficiently small constant $c>0$. Then gradient descent  
for problem (\ref{obj}) (i.e. $K=1$) has
$ \textstyle
T_{\gamma} \leq 
\frac{ \log ( \frac{ \| w_* \| - \gamma  }{ |w_0[1]| }  ) }{ \log( 1 + \eta \Delta) },
$
where $\Delta := 6 \gamma ( \| w_* \| - \gamma ) > 0$.
Furthermore, for $0 \leq t \leq T_{\gamma}$, we have that
$| w_t^\parallel | \geq (1 + \eta \Delta)^t | w_0^\parallel |$ and
$\| w_t^\perp \| \leq \gamma$.
\end{theorem}
Theorem~\ref{lem:single} states that to achieve $\dist^2( w_t, w_*) \leq 2 \gamma^2$,
gradient descent only needs at most $\log ( \frac{ \| w_* \| - \gamma  }{ |w_0[1]| }  ) / \log( 1 + \eta \Delta) $ number of iterations.
Furthermore, on the signal direction, $|w_t[1]|$ (and hence $w_t^\parallel$) grows exponentially at a rate at least $1+\eta \Delta$ before reaching at $\| w_* \| - \gamma$. On the other hand, by a close-to-zero initialization (C2), the perpendicular component $\| w^\perp \|$ remains small.
Consequently, we have that $\dist^2(W^{\#1}_t, w_*) \leq \max( \gamma^2, \big( | w_* | -| w_0[1] | (1+\eta \Delta)^t  \big)^2 ) + \gamma^2$, for $0\leq t \leq T_{\gamma}$ before gradient descent enters the linear convergence regime.
The proof of Theorem~\ref{lem:single} is available in Appendix~\ref{app:single}. A similar result as Theorem~\ref{lem:single} was shown in \citep{CCFMY18}.

\subsection{How does over-parametrization help entering a benign region?}

Let us begin by providing a more detailed observation regarding the dynamics of gradient descent.
Figure~\ref{fig:component} plots each component of $\tparw{k}{K}{t}$ and $\tpenw{k}{K}{t}$ in the gradient descent dynamics (\ref{eq:multi-dyn}), while
Figure~\ref{fig:componentRatio} plots
(the magnitude of) component B to component A and the ratio of component D to component C for $K=10$.
The empirical findings show that
the component due to the interaction (component B, or component D respectively) is negligible compared to 
the component that is without the dependency on the other neurons (component A, or component C respectively).
Based on the observation, 
we can re-write the population dynamics (\ref{eq:multi-dyn}) in the early stage (i.e. before gradient descent enters the linear convergence regime) as follows,
\begin{equation} \label{eq:multi-dyn2}
\begin{aligned}
\textstyle  | \tparw{k}{K}{t+1} |
& \textstyle  \geq  (1 - \theta)
| \tparw{k}{K}{t} |  \big( 1 + \eta ( 3 \| w_* \|^2 - 3\| \tw{k}{K}{t} \|^2  ) \big) 
\\ \textstyle 
\| \tpenw{k}{K}{t+1} \|
& \textstyle  \leq (1 + \vartheta)  
\| \tpenw{k}{K}{t} \| \big( 1 + \eta ( \| w_* \|^2  -3 \| \tw{k}{K}{t} \|^2  ) \big) .
\end{aligned}
\end{equation}
for some small numbers $\theta, \vartheta \ll 1$ (empirically $\approxeq 10^{-4}$ as Figure~\ref{fig:component} shows).
\begin{figure}[t]
\centering
     \subfigure[Component A vs. $t$. \label{subfig-1:dummy}]{%
       \includegraphics[width=0.35\textwidth]{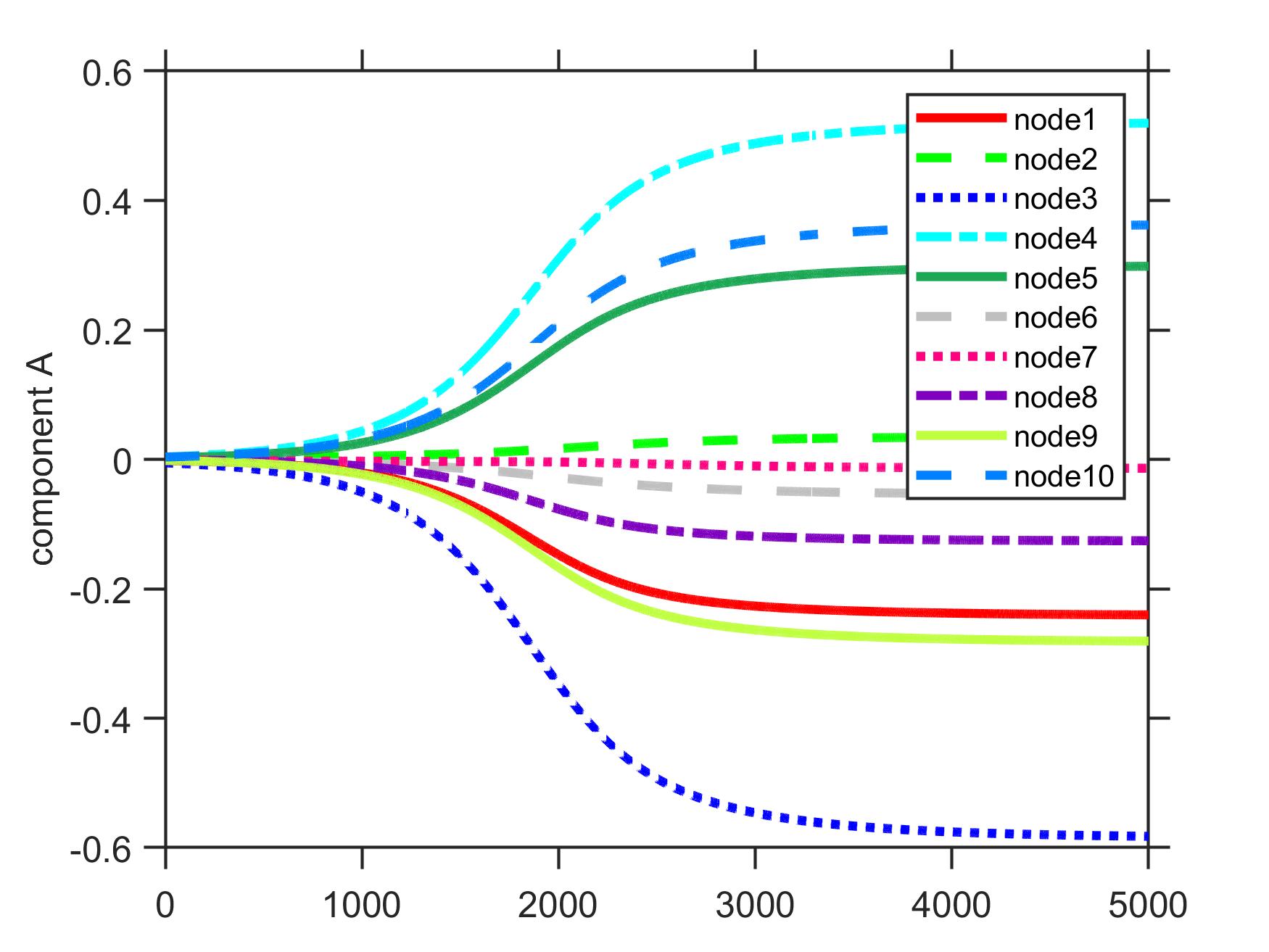}
     } 
     \subfigure[Component B vs. $t$.\label{subfig-2:dummy}]{%
       \includegraphics[width=0.35\textwidth]{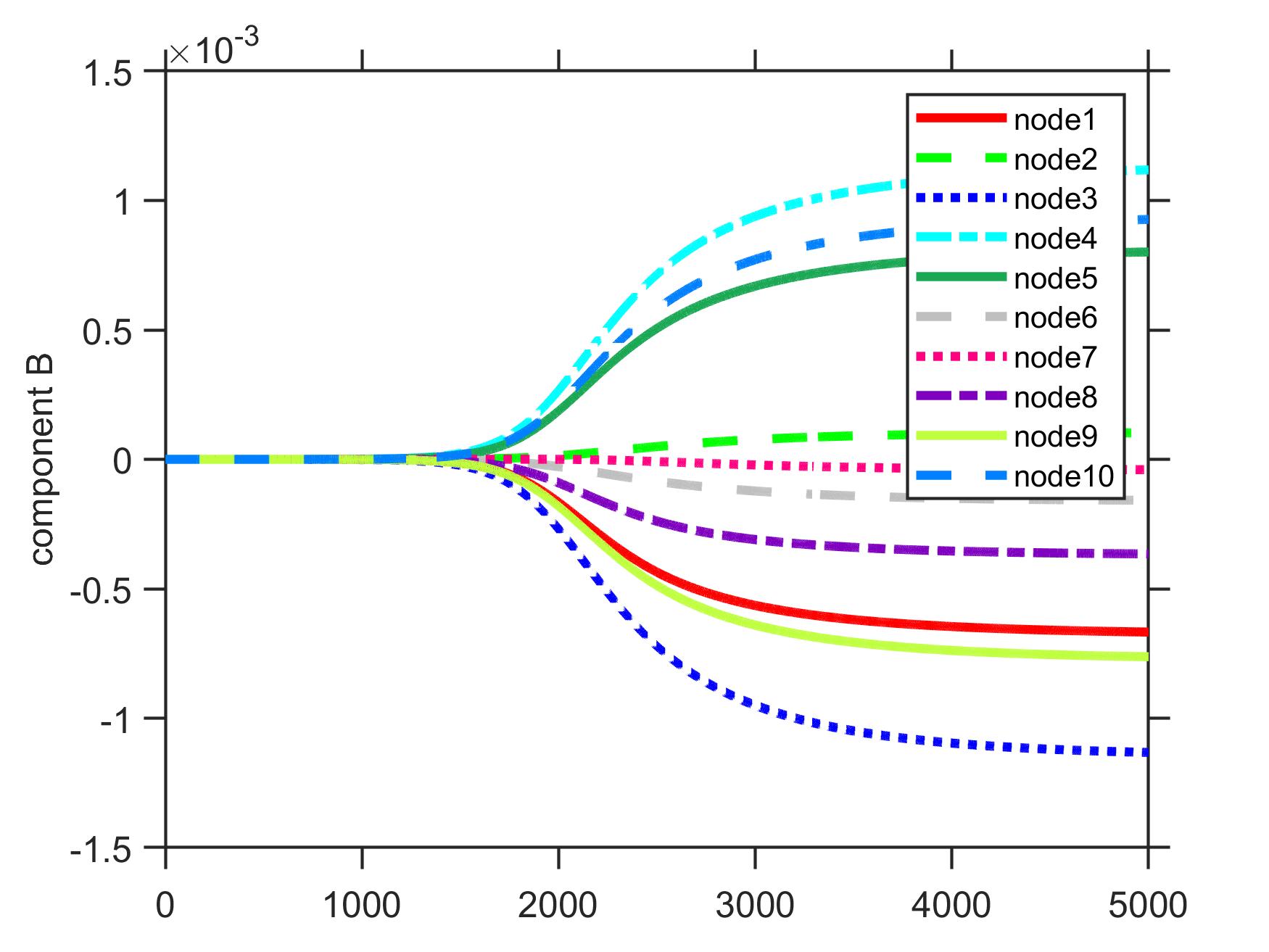}
     } \\
     \subfigure[Component C vs. $t$. \label{subfig-1:dummy}]{%
       \includegraphics[width=0.35\textwidth]{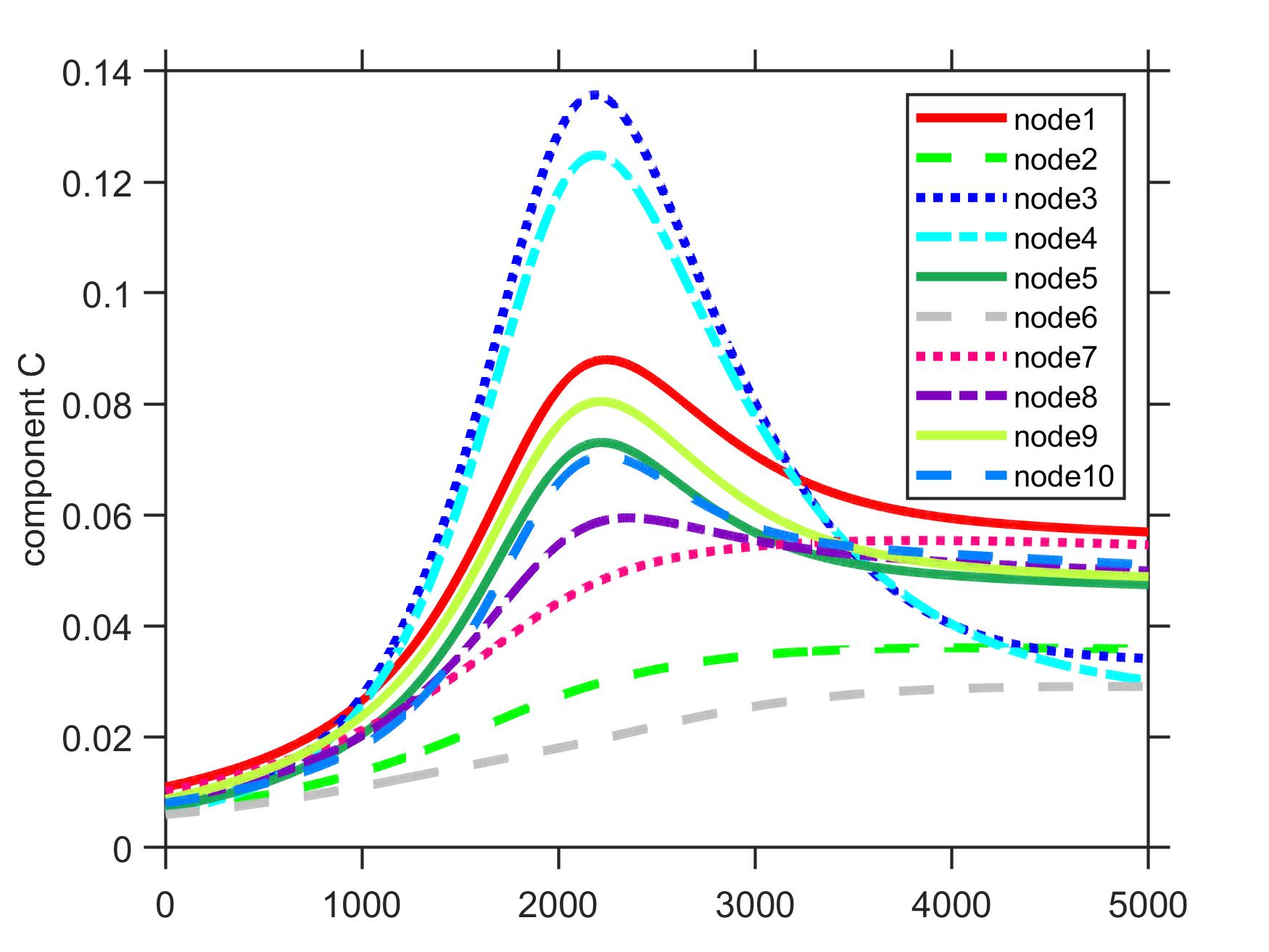}
     } 
     \subfigure[Component D vs. $t$.\label{subfig-2:dummy}]{%
       \includegraphics[width=0.35\textwidth]{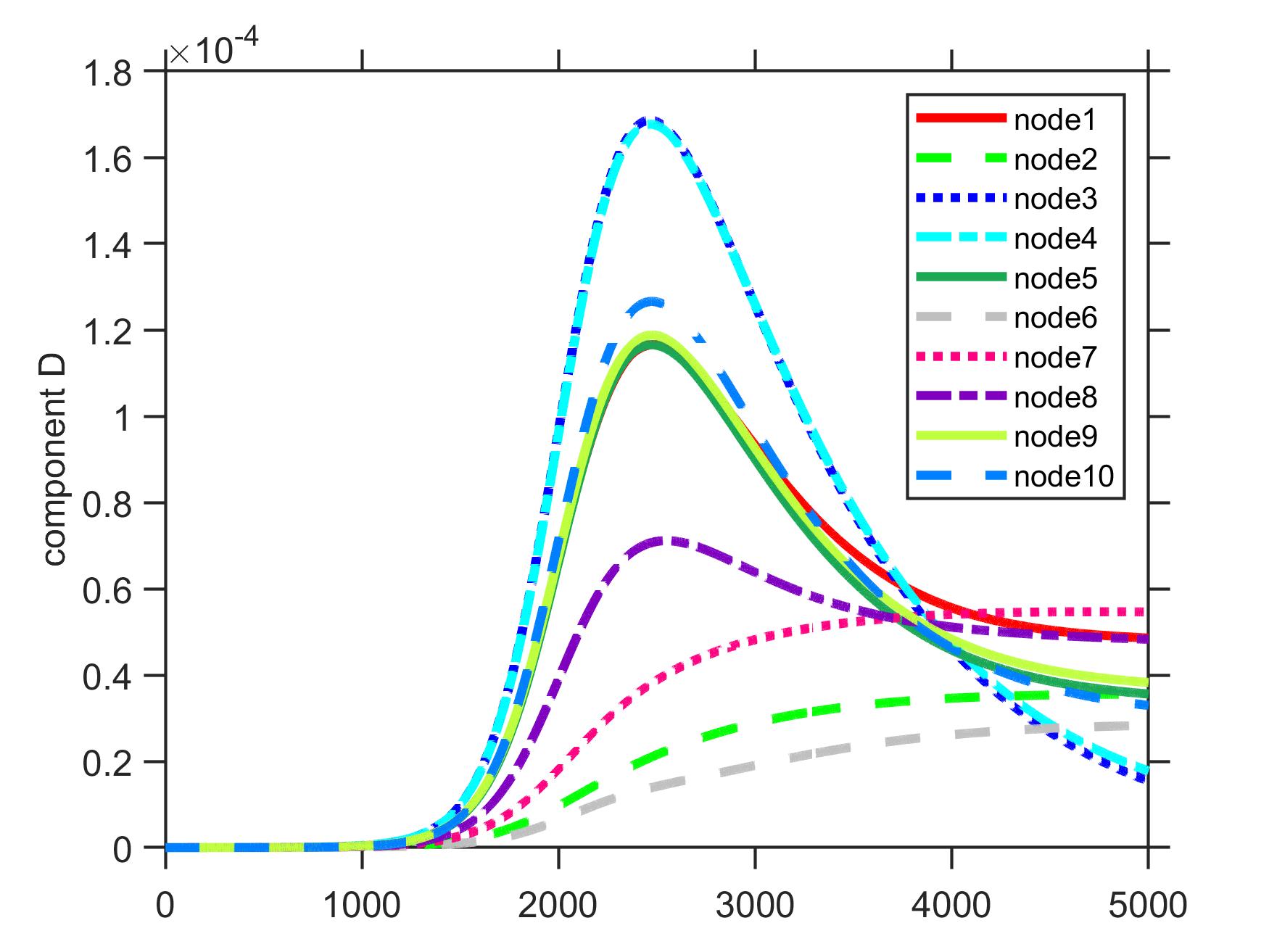}
     }
     \caption{\footnotesize 
Subfigure (a) shows
$\parw{k}_t  \big( 1 + \eta ( 3 \| w_* \|^2 - 3\| \w{k}_t \|^2  ) \big)$ 
(i.e. component A of $\tparw{k}{K}{t}$)
versus iteration $t$ for each neuron $k$.
Subfigure (b) plots $ 2 \eta \sum_{j \neq k }^K  ( (\w{j}_t)^\top \w{k}_t  )  \parw{j}_t
+ \eta \parw{k}_t \sum_{j \neq k}^K \| \w{j}_t \|^2$ 
(i.e. component B of $\tparw{k}{K}{t}$)
versus iteration $t$ for each neuron $k$. 
Subfigure (c) plots the norm of component C of $\tpenw{k}{K}{t}$, while subfigure (d) plots the norm of component D of $\tpenw{k}{K}{t}$ for each neuron $k$.
The empirical findings show that the components due to interaction of the other neurons (i.e. component B and D) are small (notice that the scale of the vertical axis
of (a) and (b), (c) and (d) are different) compared to their counterparts (i.e. component A and C respectively), which suggests that $\theta, \vartheta \approxeq 10^{-4}$ on (\ref{eq:multi-dyn2}) empirically. 
Similar patterns exhibit in training under different $K$'s.
      }
     \label{fig:component}
\end{figure}

\begin{figure}[h]
  \centering
    \includegraphics[width=0.4\textwidth]{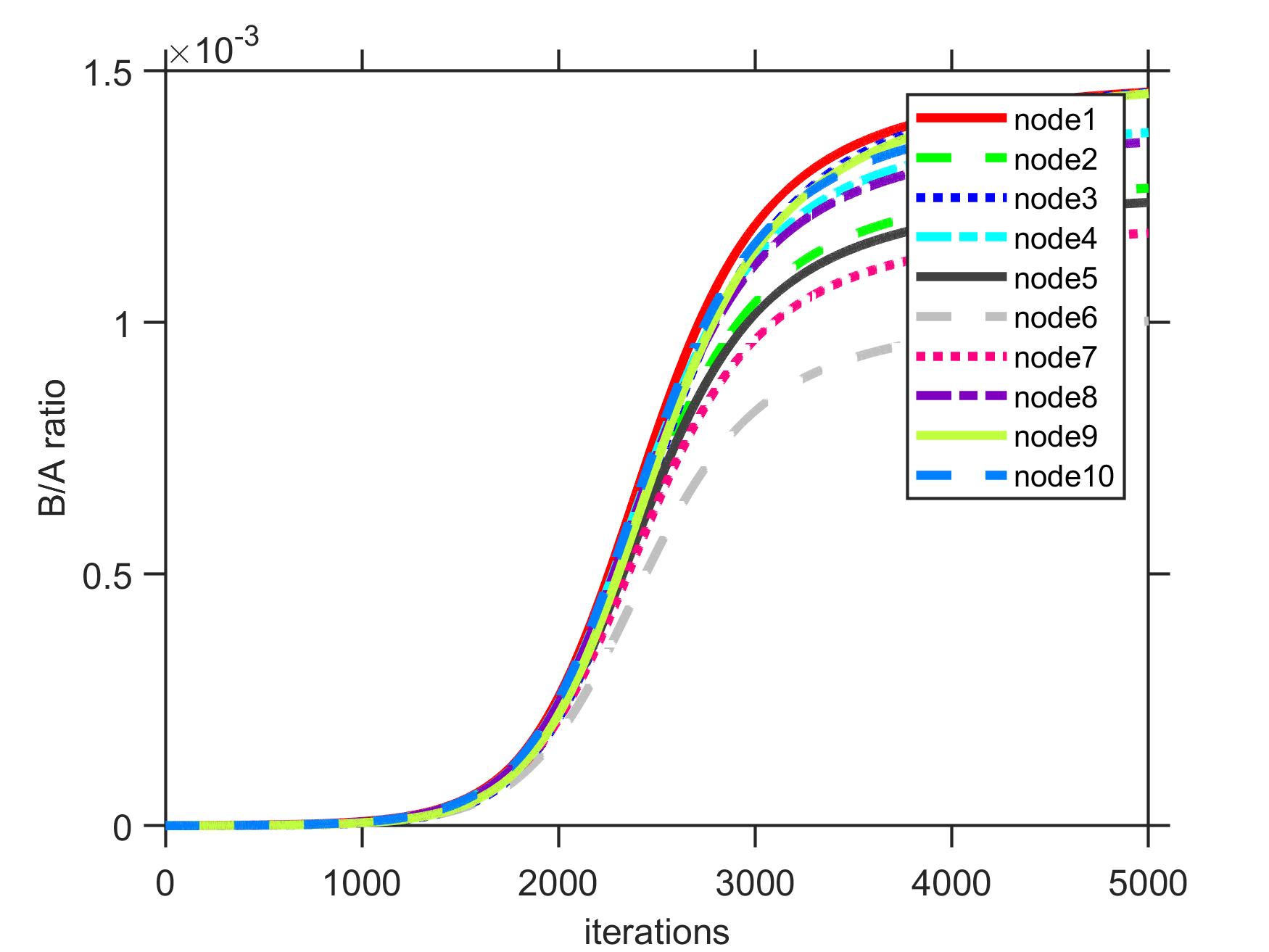}
   \includegraphics[width=0.4\textwidth]{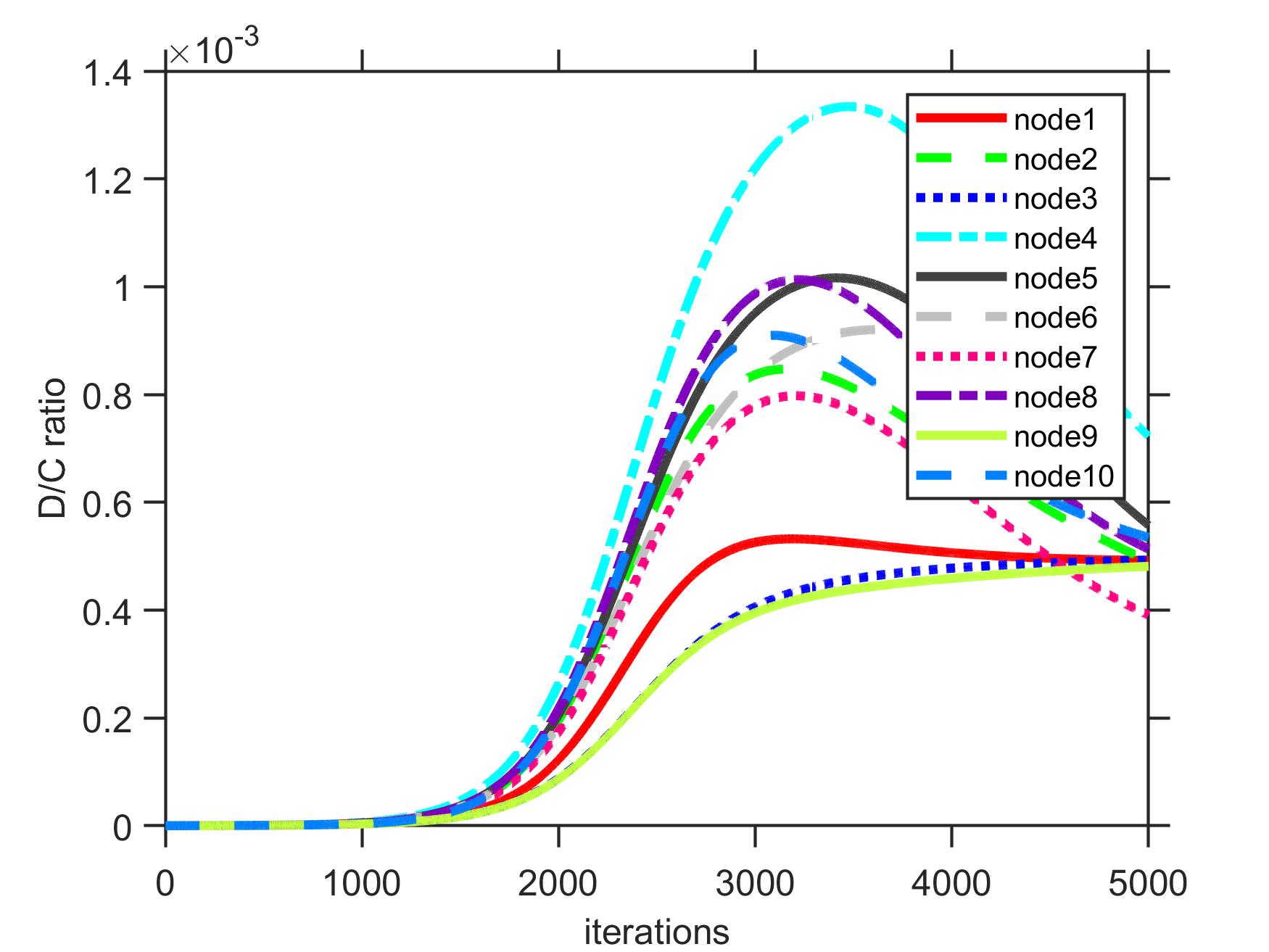}
    \caption{\footnotesize The ratio of (the magnitude of) component B to component A and the ratio of component D to component C.
    The plot shows that the components due to the interaction (component B and component D) are negligible.
    } \label{fig:componentRatio}   
\vspace{-2.5mm}
\end{figure}

That is, in the early stage, the approximated dynamics (\ref{eq:multi-dyn2}) of each student neuron of an over-parametrized network does not deviate too much from the original dynamics without over-parametrization (\ref{dyn:single}). 
The approximated dynamics (\ref{eq:multi-dyn2}) will be only
used for analyzing the stage before the iterate enters the benign region, i.e. used only before entering a linear convergence regime,
 though empirically the approximated dynamics hold all the time during the execution of the algorithm (Figure~\ref{fig:component}).

On the other hand, 
by comparing subfigure (a) and subfigure (b) of Figure~\ref{fig:component}, 
we see that component A and B of each neuron $k$ are with the same sign during the execution.
This observation together with the dynamics (\ref{eq:multi-dyn}) tend to
implies that
$\textstyle  | \tparw{k}{K}{t+1} |
\textstyle  \leq  | \tparw{k}{K}{t} |  \big( 1 + \eta ( 3 \| w_* \|^2 - 3\| \tw{k}{K}{t} \|^2  ) \big) 
$ when $K>1$. On the other hand, the dynamics (\ref{dyn:single}) has
$| \tparw{1}{1}{t+1} | = | \tparw{1}{1}{t} |  \big( 1 + \eta ( 3 \| w_* \|^2 - 3\| \tw{1}{1}{t} \|^2  ) \big)$. Also, 
as the case of single neuron,
the perpendicular component of each neuron $k$ of $K$ (i.e. $\| \tpenw{k}{K}{t} \|$) remains small (Figures~\ref{fig:perpen-app}).
Consequently, we could write
\begin{equation} \label{multi-par}
\textstyle |\tparw{1}{1}{t}| \gtrsim |\tparw{k}{K}{t}| \text{ and }
\textstyle \| \tw{1}{1}{t} \| \gtrsim  \| \tw{k}{K}{t} \|, 
\end{equation}
where the approximation $\gtrsim$ accounts for the fact that the size of initial points due to the random initialization may be different
and that the small interaction components are present in the dynamics for $K>1$. 
Figure~\ref{fig:breakdown} confirms that the relation (\ref{multi-par}) generally holds on average empirically. 

\begin{figure}[h]
\centering
     \subfigure[\footnotesize $K=1$. \label{subfig-1:dummy}]{%
       \includegraphics[width=0.3\textwidth]{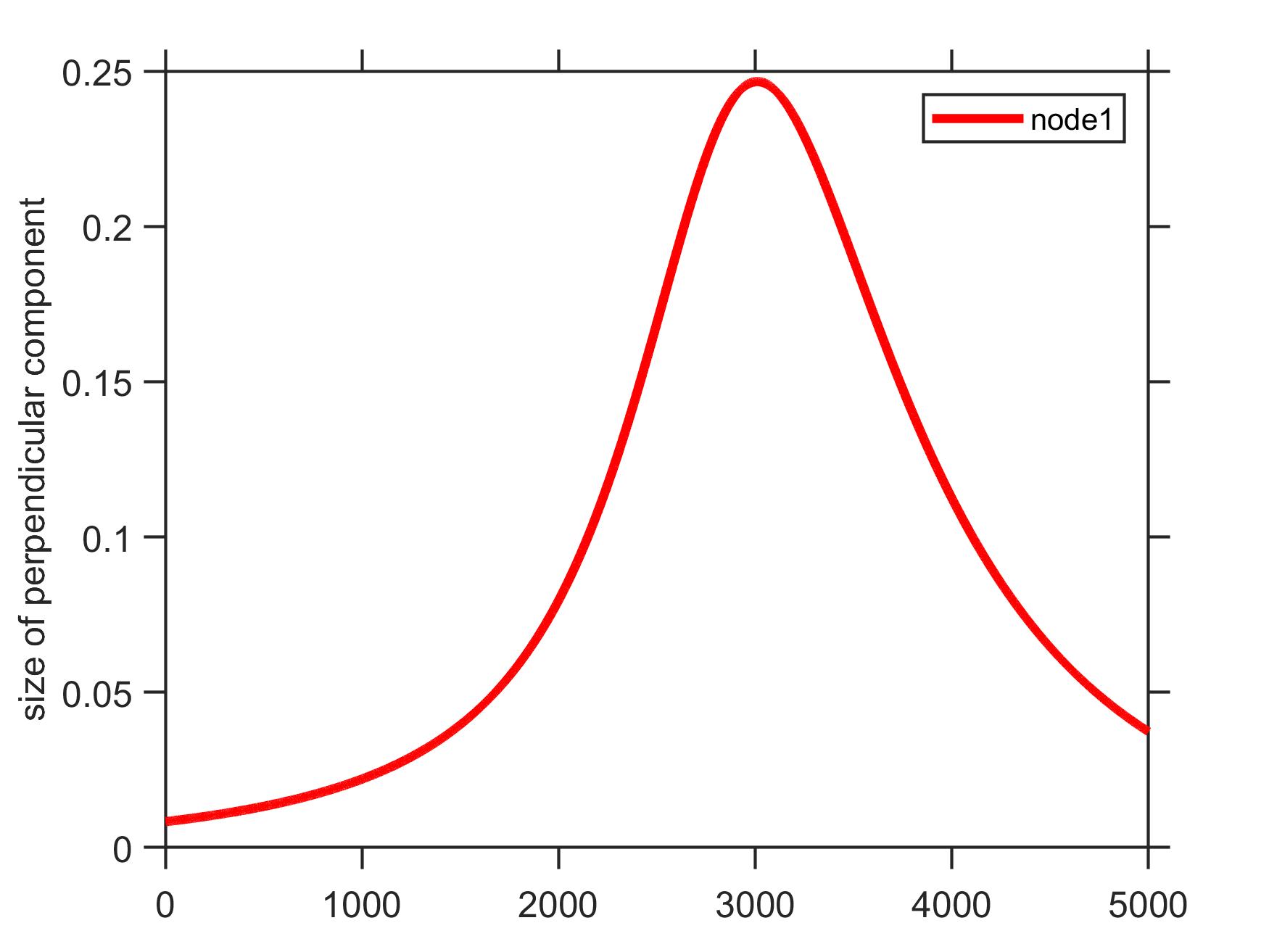}
     } 
     \subfigure[\footnotesize  $K=3$.\label{subfig-2:dummy}]{%
       \includegraphics[width=0.3\textwidth]{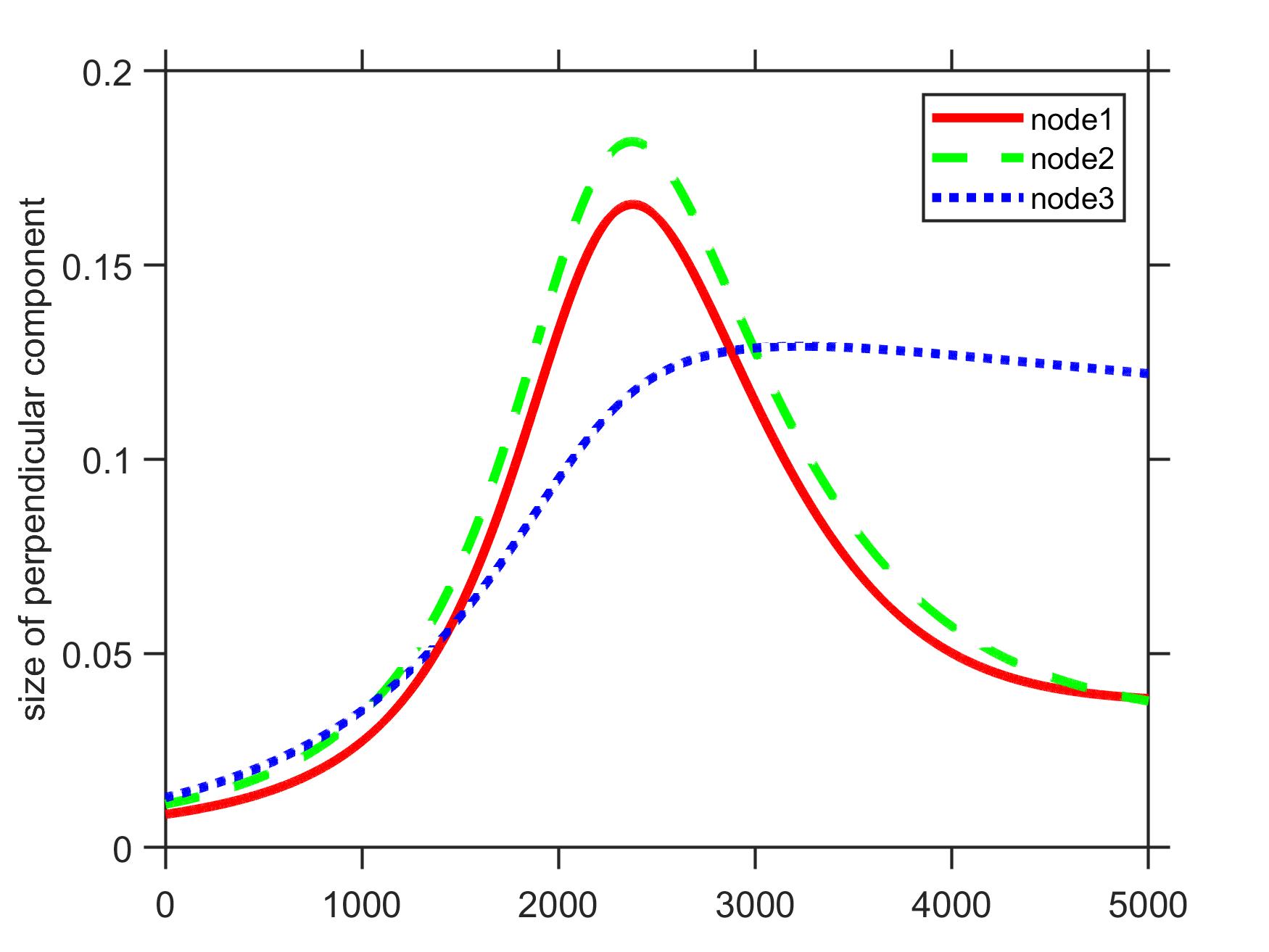}
     }
     \subfigure[\footnotesize $K=10$. \label{subfig-1:dummy}]{%
       \includegraphics[width=0.3\textwidth]{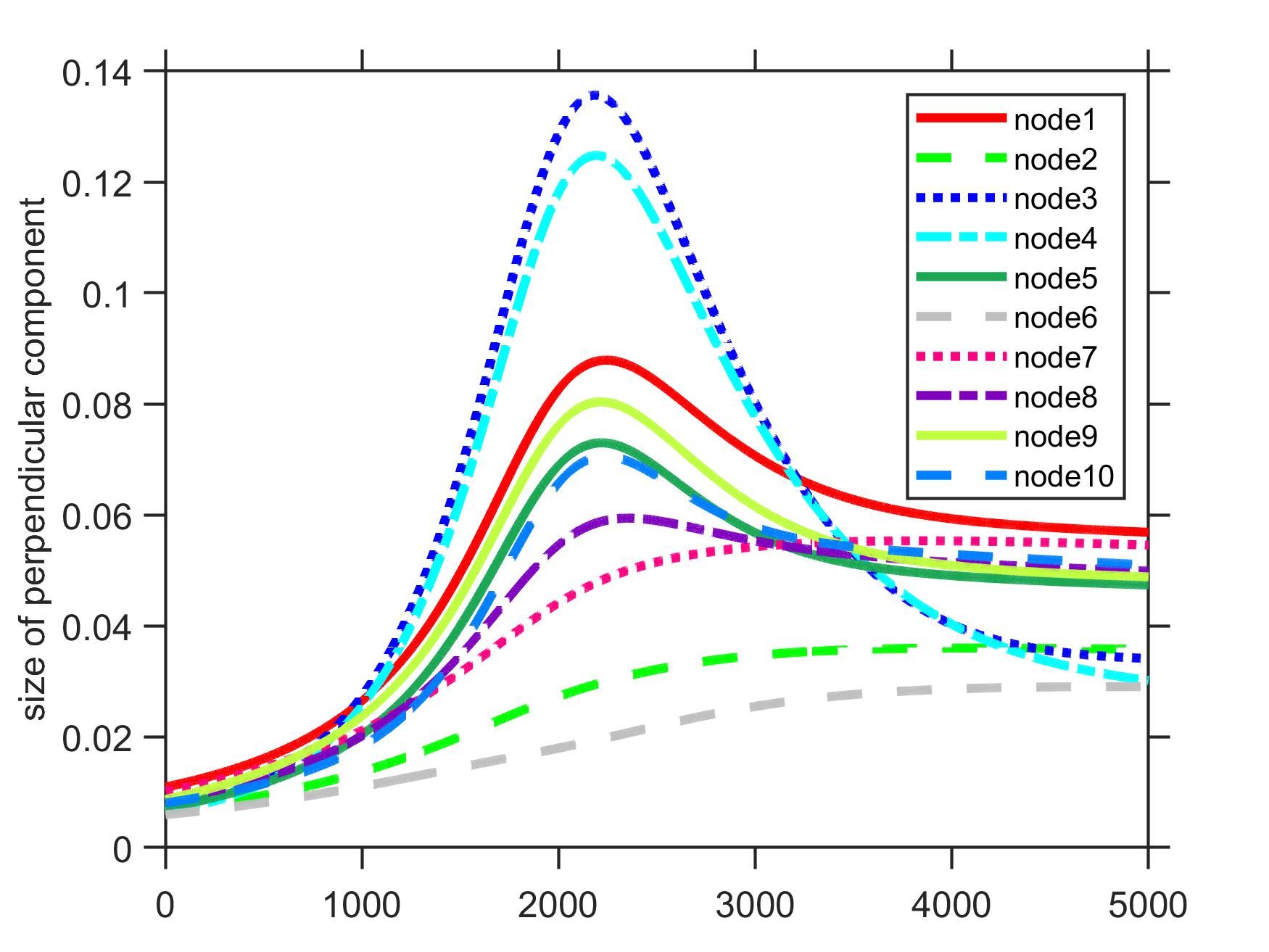}
     }
     \caption{ \footnotesize The perpendicular component 
$ \| \tpenw{k}{K}{t} \|$ of each $k$ over iterations $t$. We see that the perpendicular component remains small, compared to the signal component.
      }
     \label{fig:perpen-app}
\end{figure}

\begin{lemma} \label{lem:syn_up}
Suppose that the approximated dynamics (\ref{eq:multi-dyn2}) and (\ref{multi-par}) hold from iteration $0$ to iteration $t$.
Then, the network with a single neuron and an over-parametrized network trained by GD with the same step size $\eta$
has
$\sqrt{\sum_{k=1}^K | \tparw{k}{K}{t} |^2} \gtrsim \sqrt{ ( 1 - 2 t \theta) } \sqrt{K} | \tparw{1}{1}{t} |$.
\end{lemma}
Lemma~\ref{lem:syn_up} states that if the single neuron of the non-overparametrized network has a certain projection on $w_*$ at time $t$, then the over-parametrized network with $K$ neurons will have approximately 
$\sqrt{K}$ times larger projection on $w_*$, modulo the $\sqrt{1 - 2t \theta}$ factor which is close to $1$ if the product $t \theta$ is small (as Figure~\ref{fig:component} shows).
This demonstrates the advantage of over-parametrization --- over-parametrization helps to make more progress on growing the model's projection on $w_*$.

\begin{lemma} \label{lem:norm_up}
Suppose that $\eta \leq \frac{1}{3 \| w_* \|^2}$.
By following the conditions as Lemma~\ref{lem:syn_up},
we have that
$
\| \tpenw{k}{K}{t} \|
\lesssim \frac{ |\tparw{k}{K}{t}| \| \tpenw{k}{K}{0} \|  }{ | \tparw{k}{K}{0} |   } \frac{1}{\psi^t}
\lesssim 
\frac{ | \tparw{1}{1}{t} | \| \tpenw{k}{K}{0} \|  }{ | \tparw{k}{K}{0} |   } \frac{1}{\psi^t},$
where $\psi:= \big(1- \theta - \vartheta - \theta \vartheta  \big) \big( 1 + \eta \| w_* \|^2  \big)$.
\end{lemma}
Lemma~\ref{lem:norm_up} states that the ratio of the perpendicular component $\| \tpenw{k}{K}{t} \|$ to the parallel component $|\tparw{k}{K}{t}|$ of each neuron decays exponentially if $\psi > 1$ (which holds if $\eta \| w_* \|^2 \gtrsim \frac{\theta + \vartheta}{ 1 -  \theta - \vartheta})$.
By combining Lemma~\ref{lem:syn_up} and ~\ref{lem:norm_up}, we have the following theorem, which characterizes the difference of the distances to $w_*$
at iteration $t$.

\begin{theorem} \label{lem:key}
(Snapshot at $t$)
Suppose that the approximated dynamics (\ref{eq:multi-dyn2}) and (\ref{multi-par}) hold from $0$ to $t$
and that at iteration $t$, the student network with a single neuron trained by GD with the step size $\eta$ has $| \tparw{1}{1}{t} | = c_{1,t} \| w_* \|^2$ for some number $c_{1,t}$ satisfying $1 > c_{1,t} > 0$.
Denote $c_{2,t}$ a number that satisfies $\frac{ c_{1,t}  \| \tpenw{k}{K}{0} \|  }{ | \tw{k}{K}{0}[1] | } \frac{1}{\psi^t}
\leq c_{2,t}$ for each $k \in [K]$.
Suppose that the step size $\eta$ also satisfies
$\eta \leq \frac{1}{3 \| w_* \|^2}$
and makes
$\psi:= \big(1- \theta - \vartheta - \theta \vartheta  \big) \big( 1 + \eta \| w_* \|^2  \big) > 1$.
Then, an over-parametrized network $W_t^{\#K}$ trained by GD with the same $\eta$ has
\[
\textstyle
\dist^2( W_t^{\#1}, w_*) - \dist^2( W_t^{\#K}, w_*) 
\gtrsim
\| w_* \|^2 \big( 2  c_{1,t}( \sqrt{ ( 1 -2 t \theta) }  \sqrt{K} - 1 )   
- K (c_{1,t}^2 + c_{2,t}^2) + c_{1,t}^2 \big).
\]
\end{theorem}

Recall that in Theorem~\ref{lem:single}, we upper-bound 
$\dist^2( W_t^{\#1}, w_*)$.
Simply combining Theorem~\ref{lem:single} and ~\ref{lem:key} leads to a distance upper-bound of $\dist^2( W_t^{\#K}, w_*)$ at certain iteration $t$.
The lower bound of the difference of the distances in Theorem~\ref{lem:key} shows a strict improvement due to over-parametrization when it is positive, 
which answers why over-parametrization helps gradient descent to enter the linear convergence regime faster ---
gradient descent for an over-parametrized network shrinks the distance to $w_*$ faster in the early stage.
Note that the lower bound is a quadratic function of $\sqrt{K}$ and is increasing for 
$1 \leq \sqrt{K} \leq \frac{c_{1,t} \sqrt{ 1 -2 t \theta } }{ c_{1,t}^2 + c_{2,t}^2}$, which means that up to a certain threshold of $K$, more over-parametrization could lead to more improvements. Moreover, if $c_{1,t}$ and $c_{2,t}$ further satisfy 
$\textstyle (*) \text{ }
2 c_{1,t} \big( \sqrt{2} \sqrt{ (1-2t\theta)} -1 \big) - c_{1,t}^2 - 2 c_{2,t}^2 > 0,$
then the lower bound of the difference for $K=2$ neurons is strictly positive and keeps being positive up to $\sqrt{K} \leq \lfloor \frac{ c_{1,t} \sqrt{1-2 t \theta} + \sqrt{ c_{1,t}^2 (1 - 2 t \theta) - (c_{1,t}^2 + c_{2,t}^2) (2 c_{1,t} - c_{1,t}^2) } }{ c_{1,t}^2 + c_{2,t}^2} \rfloor$, which gives an upper limit of the degree of over-parametrization that allows acceleration.
The condition (*) is easily satisfied when (1) $t \theta \ll 1$ so that $\sqrt{(1-2t\theta)} \approxeq 1$ and (2) $c_{2,t} \ll 1$, which happens when the approximated dynamics (\ref{eq:multi-dyn}) holds for a small $\theta$ and that the ratio of the perpendicular component to the parallel component of neuron $k$ decays sufficiently fast (Lemma~\ref{lem:norm_up}). 


\begin{proof}[of Theorem~\ref{lem:key}]

Let us compute
$\dist( W_t^{\#1}, w_*)$ and $\dist( W_t^{\#k}, w_*)$.
For $\dist( W_t^{\#1}, w_*)^2$, we have that
\begin{equation} \label{eq:t0}
\begin{split}
\dist( W_t^{\#1}, w_*)^2 & = \min_{q \in \{-1, +1\} }  \|  W_t^{\#1} - w_*q^\top \|^2
= \| \tw{1}{1}{t} \|^2 - 2 | \tparw{1}{1}{t} | + \| w_* \|^2
\\ & = \| \tw{1}{1}{t} \|^2 + \| w_* \|^2 - 2 \sqrt{\| \tparw{1}{1}{t} \|^2} .
\end{split}
\end{equation}
On the other hand, for $\dist^2( W_t^{\#K}, w_*)$, we have that
\begin{equation} \label{eq:t1}
\begin{split}
\dist^2( W_t^{\#K}, w_*) & = \min_{q \in \reals^K: \| q\|_2 \leq 1}  \|  W_t^{\#K} - w_*q^\top \|^2_F
\\ & 
= \min_{q \in \reals^K: \| q\|_2 \leq 1}  
\tr\big(  (W_t^{\#K} - w_* q^\top)^\top (W_t^{\#K} - w_* q^\top)  \big)
\\ & = \| W_t^{\#K} \|^2_F + \| w_* \|^2 - 2 \max_{q \in \reals^K: \| q\|_2 \leq 1} 
\tr\big(  (W_t^{\#K})^\top w_* q^\top \big)
\\ & 
= \| W_t^{\#K} \|^2_F + \| w_* \|^2 - 2 \sqrt{\sum_{k=1}^K | \tparw{k}{K}{t} |^2},
\end{split}
\end{equation}
where the last inequality is due to that
\begin{equation} \label{eq:t2}
\begin{split}
& \max_{q \in \reals^K: \| q\|_2 \leq 1} \tr\big(  (W_t^{\#K})^\top w_* q^\top \big) 
\overset{(a)}{=}  \max_{q \in \reals^K: \| q\|_2 \leq 1} \tr\big( 
\bar{v}
q^\top \big)
 =  \max_{q \in \reals^K: \| q\|_2 \leq 1} \tr\big( 
q^\top \bar{v}
 \big)
\overset{(b)}{=} \| \bar{v} \|. 
\end{split}
\end{equation}
where $(a)$ we denote $\bar{v} := [ \tparw{1}{K}{t}, \tparw{2}{K}{t}, \dots, \tparw{K}{K}{t} ]^\top \in \reals^K$ and (b) is because $q = \bar{v} / \| \bar{v} \|$. 
Combining (\ref{eq:t0}) and (\ref{eq:t1}), we have that
\begin{equation} \label{eq:?}
\begin{split}
& \dist^2( W_t^{\#1}, w_*) - \dist^2( W_t^{\#K}, w_*) 
\\ & =
2 ( \sqrt{\sum_{k=1}^K | \tparw{k}{K}{t} |^2} - \sqrt{ \| \tparw{1}{1}{t} \|^2 } ) - 
\big(  \| W_t^{\#K} \|^2_F - \| W_t^{\#1} \|^2_F  ).
\end{split}
\end{equation}
To continue, we need the lower bound of $\sqrt{\sum_{k=1}^K | \tparw{k}{K}{t} |^2} - \sqrt{ \| \tparw{1}{1}{t} \|^2 }$
and the upper bound of $\| W_t^{\#K} \|^2_F - \| W_t^{\#1} \|^2_F$.
By Lemma~\ref{lem:syn_up}, we have that
\begin{equation} \label{eq:lo}
\sqrt{\sum_{k=1}^K | \tparw{k}{K}{t} |^2} - \sqrt{ \| \tparw{1}{1}{t} \|^2 } \gtrsim (\sqrt{ ( 1 -2 t \theta) } \sqrt{K} - 1) \| \tparw{1}{1}{t} \|
= (\sqrt{ ( 1 -2 t \theta) } \sqrt{K} - 1) c_{1,t} \| w_* \|^2.
\end{equation}
On the other hand, for the difference of the norms,
we have that $\| W_t^{\#1} \|^2 \geq | \tw{1}{1}{t}(1) |^2 =  \frac{ | \tparw{1}{1}{t} |^2 }{ \| w_* \|^2  } $.
Furthermore,
by Lemma~\ref{lem:norm_up} and that $| \tparw{1}{1}{t} | = c_{1,t} \| w_* \|^2$ and the definition of $c_{2,t}$,
$ \| \tpenw{k}{K}{t} \|_2 \lesssim
\frac{ |\tparw{1}{1}{t}| \| \tpenw{k}{K}{0} \|  }{ | \tparw{k}{K}{0} |   } \frac{1}{\psi^t} = \frac{ |\tparw{1}{1}{t}| \| \tpenw{k}{K}{0} \|  }{ | \tw{k}{K}{0}[1] | \| w_* \|   } \frac{1}{\psi^t} \leq c_{2,t} \| w_* \| $.
So we have that
\begin{equation} \label{eq:up}
\begin{split}
& \| W_t^{\#K} \|^2_F  = \sum_{k=1}^K \| \tw{k}{K}{t} \|^2 
= \sum_{k=1}^K \big( \| \tpenw{k}{K}{t} \|^2 + \frac{1}{\| w_* \|^2} | \tparw{k}{K}{t} |^2 \big)
\\ & 
\lesssim
\sum_{k=1}^K \big( \big( \frac{ |\tparw{1}{1}{t}| \| \tpenw{k}{K}{0} \|  }{ | \tw{k}{K}{0}[1] | \| w_* \|   } \frac{1}{\psi^t}  \big)^2  + \frac{1}{\| w_* \|^2} | \tparw{k}{K}{t} |^2 \big)
\leq
K \| w_* \|^2 c_{2,t}^2 
+  \sum_{k=1}^K \frac{| \tparw{k}{K}{t} |^2}{ \| w_* \|^2 }  
\\ & \lesssim K \| w_* \|^2 ( c_{2,t}^2 + c_{1,t}^2 ),
\end{split}
\end{equation}
where the last inequality uses (\ref{multi-par}).
Therefore, by combining (\ref{eq:?},\ref{eq:lo},\ref{eq:up}),
\begin{equation} \label{dist:diff}
\begin{split}
& \dist^2( W_t^{\#1}, w_*) - \dist^2( W_t^{\#K}, w_*) 
\gtrsim 
\| w_* \|^2 \big( 2  c_{1,t}( \sqrt{ ( 1 -2 t \theta) }  \sqrt{K} - 1 )   
- K (c_{1,t}^2 + c_{2,t}^2) + c_{1,t}^2 \big).
\end{split}
\end{equation}

\end{proof}

\subsection{Is over-parametrization equivalent to using a larger step size?} \label{app:diff}



Following the simulation as Figure~\ref{fig:trts}, 
we tried different values of step sizes 
\[
\eta = \{1.0, 0.5, 0.4, 0.3, 0.2, 0.1, 0.05, 0.01, 0.005, 0.001\}
\]
for each model with different number of neurons $K= \{1,3,10\}$.
We report 
the quantity (\ref{eq:vHv}) over iteration $t$,
\[ 
\textstyle
\vec( w_* q_t^\top - W^{\#K}_t)^\top \nabla^2 f(W^{\#K}_t) \vec(w_* q_t^\top - W^{\#K}_t),
\]
 where $\nabla^2 f(W^{\#K}_t) \in \reals^{dK\times dK }$ is the Hessian and $w_* q_t^\top$ is the closet global optimal solution to $W^{\#K}_t$ and
the notation $\vec(\cdot)$ represents the vectorization operation of its matrix argument. 
Recall that the quantity can be viewed as a measure of the strong convexity as mentioned in the main text. 
Specifically, if the quantity is larger than $0$, then it suggests that the current optimization landscape is strongly convex with respect to 
$w_* q_t^\top$.

We found out that gradient descent with step size $\eta = \{1.0, 0.5, 0.4, 0.3\}$ either diverges or cannot converge towards zero testing error for all $K=\{1,3,10\}$, which means that $\eta = 0.2$ is basically the best step size of gradient descent for each model $K$. So the result suggests that even under optimal tuning of the step size $\eta$ for each model $K$, gradient descent for an over-parametrized model converges faster than the case for a smaller model. We show some results of using different $\eta$ on
Figure~\ref{fig:optimal-tuning}, Figure~\ref{fig:cannot-converge}, and Figure~\ref{fig:etascale}. Based on the empirical results, we conclude that the acceleration due to over-parametrization cannot be simply explained by that gradient descent uses a larger \emph{effective} step size, as the impacts due to parameters $\eta$ and $K$ seem to be complementary in the experiment.

\begin{figure}[h]
\centering
     \subfigure[\footnotesize ($\eta=0.2$) quantity (\ref{eq:vHv}) over $t$. \label{subfig-1:dummy}]{%
       \includegraphics[width=0.3\textwidth]{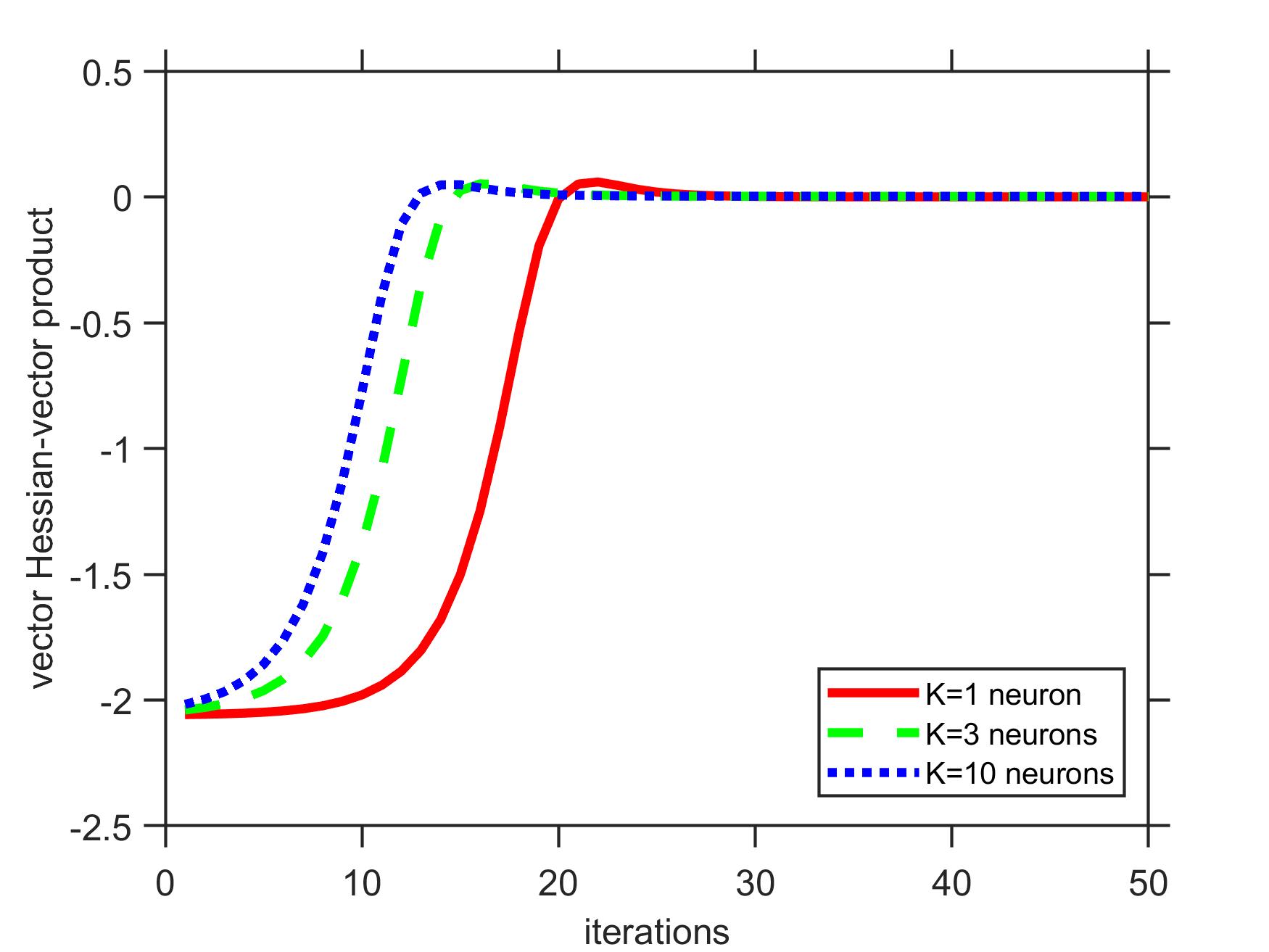}
     }  
     \subfigure[\footnotesize  ($\eta=0.05$) quantity (\ref{eq:vHv})  over $t$. \label{subfig-1:dummy}]{%
       \includegraphics[width=0.3\textwidth]{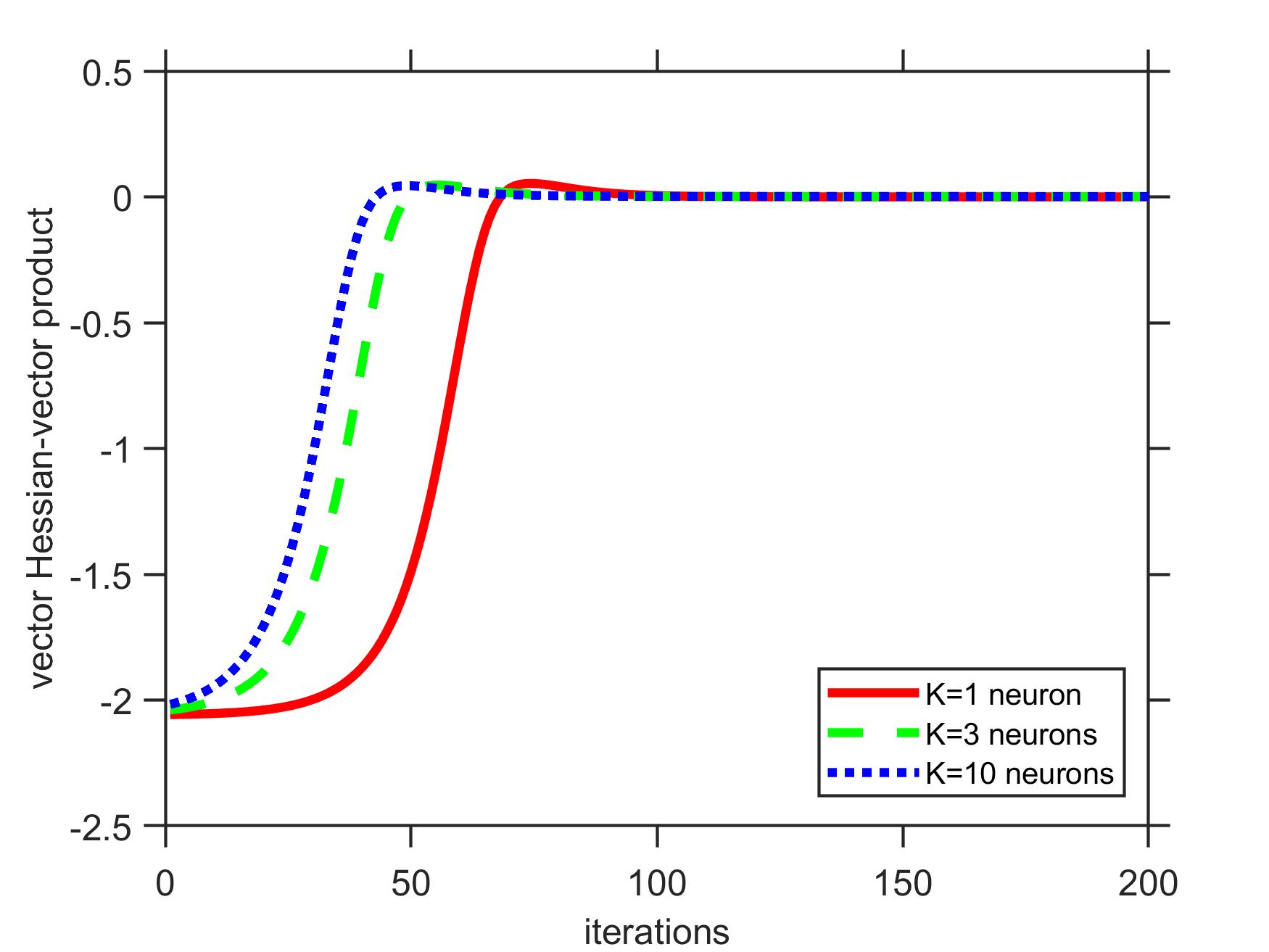}
      } 
     \subfigure[\footnotesize  ($\eta=0.01$) quantity (\ref{eq:vHv})  over $t$. \label{subfig-1:dummy}]{%
       \includegraphics[width=0.3\textwidth]{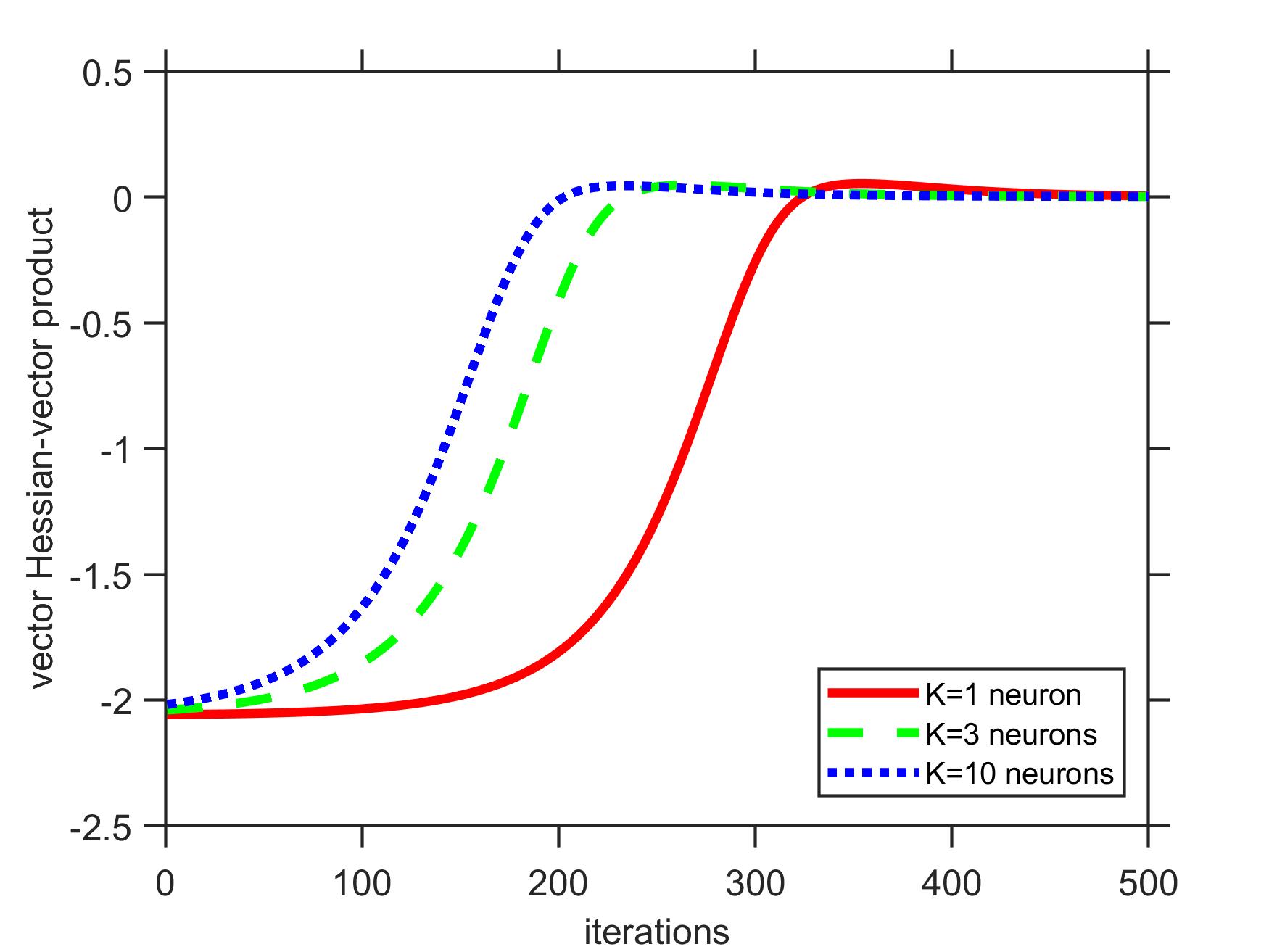}
      }
     \caption{  \footnotesize
     Gradient descent with different values of the step size.
     Note that the scales of the horizontal axes are different.  
     Both the step size $\eta$ and the degree of over-parametrization affect the time that gradient descent enters the linear convergence regime. A larger step size $\eta$ and a larger number of neurons $K$ help gradient descent to make progress faster.
      }
     \label{fig:optimal-tuning}
\end{figure}


\begin{figure}[h]
\centering
     \subfigure[\footnotesize  ($\eta=0.5$) training error over $t$. \label{subfig-1:dummy}]{%
       \includegraphics[width=0.4\textwidth]{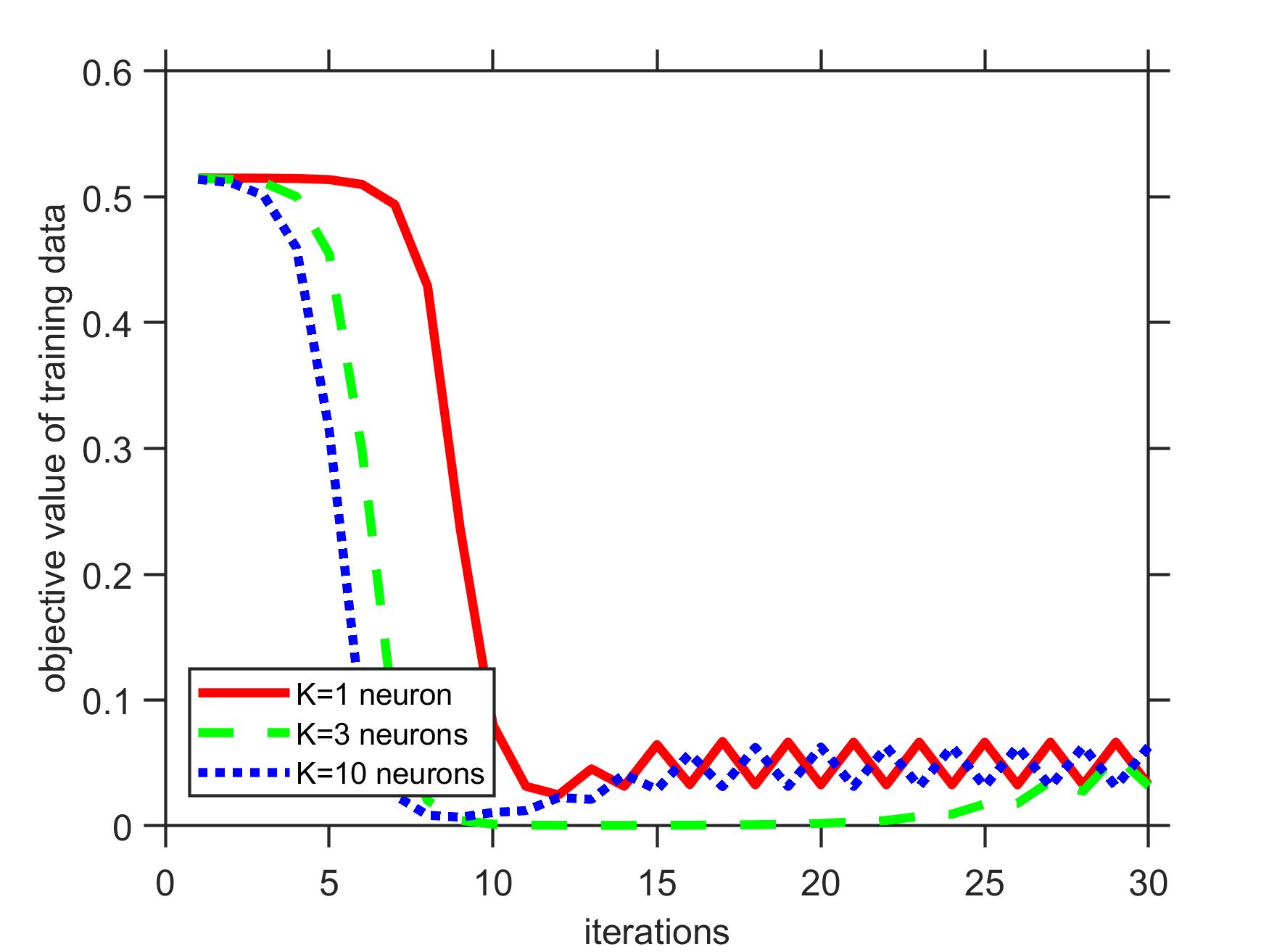}
      }  
     \subfigure[\footnotesize  ($\eta=0.5$) quantity (\ref{eq:vHv})  over $t$. \label{subfig-1:dummy}]{%
       \includegraphics[width=0.4\textwidth]{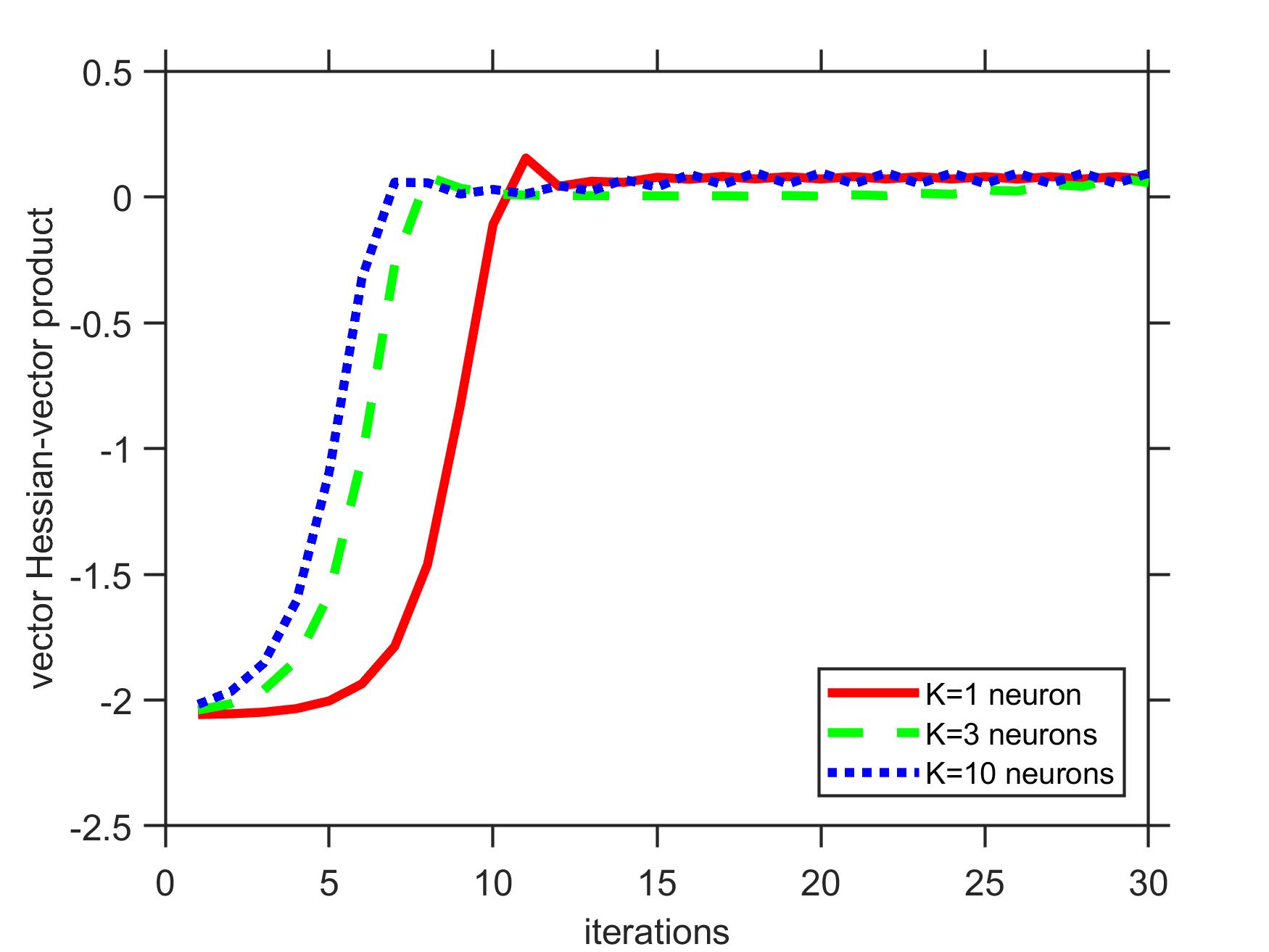}
      } \\
     \caption{ \footnotesize
     Gradient descent with $\eta = 0.5$. We see that gradient descent cannot converge to zero training (and testing) error. That is, the step size is too large to converge to a global optimal solution. Interestingly, we still observe that gradient descent requires fewer iterations to get closer to a global optimal point for a larger model, though it does not converge to a global optimal point using the large step size.
      }
     \label{fig:cannot-converge}
\end{figure}

\begin{figure}[h]
  \centering
    \includegraphics[width=0.5\textwidth]{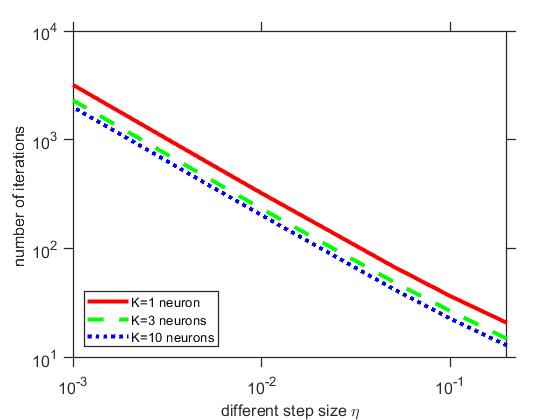}
    \caption{ \footnotesize We plot the number of iterations required for the metric
    $\textstyle
\vec( w_* q_t^\top - W^{\#K}_t)^\top \nabla^2 f(W^{\#K}_t) \vec(w_* q_t^\top - W^{\#K}_t)$ to be positive (the y-axis) under different values of the step size $\eta$ (the x-axis).  
  } \label{fig:etascale}
\end{figure}

\section{Conclusion}
We study over-parametrization for learning a single teacher nueron with quadratic activation. 
We answer why gradient descent
can achieve a faster convergence for training a larger network,
and we also show the acceleration due to over-parametrization cannot be simply explained by that gradient descent uses a larger effective step size.
We hope our work can serve as a good starting point of understanding when and why over-parametrization leads to acceleration in modern non-convex optimization. Future works include considering different activation and multiple teacher neurons.


\acks{The authors acknowledge support of NSF IIS Award 1910077. The work was performed when JW was at Georgia Tech.}

\bibliography{acml21}


\section{Proof of Lemma~\ref{lem:linear} } \label{app:benign}

Recall the optimization problem is
\[
\textstyle
 \min_{W \in \reals^{d \times K} }f(W)  :=    \frac{1}{4n} \sum_{i=1}^n \big( (x_i^\top w^{(1)})^2 + (x_i^\top w^{(2)})^2 + ... + (x_i^\top w^{(K)})^2  - y_i    \big)^2,
\]
and the notation that $\nabla F(W):= \E_{x}[ \nabla f(W) ] $ and
$\nabla^2 F(W):= \E_{x}[ \nabla^2 f(W) ] $.
We first introduce some key lemmas. 
Lemma~\ref{lem:strcvx} below says that for any $W \in \reals^{d\times K}$, if the iterate is sufficiently close to a global solution $w_* q^\top$, then the landscape is essentially strongly convex. Lemma~\ref{lem:smooth}, on the other hand, shows the smoothness of the landscape.

\begin{lemma} \label{lem:strcvx}
(locally strong convexity)
Assume that $\dist(W_{t_0}, w_*) :=\| W_{t_0} - w_* q_{t_0}^\top \|_F \leq \nu \| w_* \|_2$, where $W_{t_0} \in \reals^{d\times K}$, $\nu > 0$, $q_{t_0} := \arg\min_{q \in \reals^K: \| q\|_2 \leq 1}  \|  W_{t_0} - w_*q^\top \|_F$, and $w_* \in \reals^d$ being the teacher neuron. Then
\[ \textstyle
\vec(V)^\top \nabla^2 F(W_{t_0}) \vec(V) 
\geq 
2 \tr( q_{t_0} w_*^\top V q_{t_0} w_*^\top V  ) + (2 - 14 \nu - 2 \nu^2)  \| w_* \|^2_2 \| V \|^2_F
\]
for any $V \in \reals^{d \times K}$, where $\tr(\cdot)$ denotes the matrix trace.
\end{lemma}

\begin{proof}
The proof is a modification of the proof of Lemma 14 in \cite{LMCC19}.
For notation brevity, we will suppress script $t_0$ in the following (i.e. $W \leftarrow W_{t_0}$, $q \leftarrow q_{t_0}$).
We have that
\[
\begin{split}
& \vec(V)^\top \nabla^2 f(W) \vec(V)
\\ &
= \frac{1}{m} 
\sum_{i=1}^m \vec(V)^\top \big[ \big( ( \| x_i^\top W \|_2^2 - y_i ) I_K + 2 W^\top x_i x_i^\top W  \big) \otimes x_i x_i^\top   \big] \vec(V)
\\ &
= \frac{1}{m}
\sum_{i=1}^m \big( \| x_i^\top W \|_2^2 - y_i \big) \vec(V)^\top \vec( x_i x_i^\top V) + \frac{1}{m} \sum_{i=1}^m \vec(V)^\top \vec( 2 x_i x_i^\top V W^\top x_i x_i^\top W).
\\ &
= \frac{1}{m}
\sum_{i=1}^m \big[ \big( \| x_i^\top W \|_2^2 -  \| x_i^\top w_* q^\top \|_2^2 \big) \| x_i^\top V \|_2^2 + 2 \big( x_i^\top W V^\top x_i \big)^2  \big].
\end{split}
\]
By taking the expectation above over $x_i \sim N(0,I_d)$ and using Lemma~\ref{lem:gaussian}, we have that
\begin{equation} \label{eq:a1}
\begin{split} 
 \E[ \vec(V)^\top \nabla^2 f(W) \vec(V) ]
& = \| W \|^2_F \| V \|^2_F + 2 \| V^\top W \|^2_F 
- \big(  \| w_* q^\top \|^2_F \| V \|^2_F + 2 \| V^\top w_* q^\top \|^2_F \big)
\\ & 
+ 2\big( \tr( W^\top V )^2 + \tr( W^\top V W^\top V) + \| W V^\top \|^2_F \big).
\end{split}
\end{equation}
Now let us set $W = w_* q^\top + \theta H$ with a $H$ satisfying $\| H \|_F =1 $ and $\theta := \nu \| w_* \|$ for some number $\nu > 0$. 
\begin{equation} \label{eq:a2}
\begin{split} 
\| W \|^2_F 
& = \| w_* q^\top \|^2_F  + \theta^2 \| H \|^2_F + 2 \theta \langle w_* q^\top,  H \rangle
\\ & 
\geq \| w_* \|^2 - 2 \theta \| w_* q^\top \|_F \| H \|_F 
\\ &
\geq \| w_* \|^2 - 2 \theta \| w_* \| \| H \|_F  
\\ 
\| V^\top W \|^2_F & = \| V^\top w_* q^\top \|^2_F + \theta^2 \| V^\top H \|^2_F + 2 \theta \tr( V^\top w_* q^\top H^\top  V )
\\ & 
\geq \| V^\top w_* q^\top \|^2_F - 2 \theta \| w_* q^\top \| \| H \|_F \|V \|^2_F
\\ 
\| V W^\top \|^2_F & = \| V q w_*^\top  \|^2_F + \theta^2 \| V H^\top \|^2_F + 2 \theta \tr( V w_*^\top q H  V^\top )
\\ & \geq \| V q  w_*^\top \|^2_F - 2 \theta \| w_*^\top q \| \| H \|_F \|V \|^2_F
\\ & = \|   V q  w_*^\top  \|^2_F - 2 \theta \| w_*  \| \| H \|_F \|V \|^2_F
\\  \tr( W^\top V W^\top V)
& = \tr( q w_*^\top V q w_*^\top V  )
+ 2 \theta \tr( H^\top V q w_*^\top V )
+ \theta^2 \tr( H^\top V H^\top V)  
\\ & 
= \tr( q w_*^\top V q w_*^\top V  )
- 2 \theta \| w_* q^\top \| \| H \|_F \| V \|^2_F
- \theta^2 \| H \|^2_F \| V \|^2_F.
\end{split}
\end{equation}
Combining (\ref{eq:a1}) and (\ref{eq:a2}),
together with the bilinear property of the expectation so that
$\vec(V)^\top \nabla^2 F(w) \vec(V) = \E[ \vec(V)^\top \nabla^2 f(W) \vec(V) ]$,
 we have that
\begin{equation}
\begin{split}
\vec(V)^\top \nabla^2 F(W) \vec(V) 
 \geq & 2 \tr( q w_*^\top V q w_*^\top V  ) + 2 \|   V q  w_*^\top  \|^2_F  
\\ & - 14 \theta \| w_* \| \| H \|_F \| V \|^2_F - 2 \theta^2 \| H \|^2_F \| V \|^2_F
\\
\overset{(a)}{=} & 
2 \tr( q w_*^\top V q w_*^\top V  ) + 2  \| w_* \|^2 \| V \|^2_F  
\\ & - 14 \theta \| w_* \| \| H \|_F \| V \|^2_F -2  \theta^2 \| H \|^2_F \| V \|^2_F 
\\ = &
2 \tr( q w_*^\top V q w_*^\top V  ) + (2 - 14 \nu - 2 \nu^2)  \| w_* \|^2 \| V \|^2_F , 
\end{split}
\end{equation}
where (a) is due to that $\| A B \|^2_F \geq \sigma_r^2(A) \| B \|^2_F$ with $r$ being the rank of $A$, which in our case $A :=  w_* q^\top$ is a rank one matrix and $B := V^\top$; consequently $\sigma_1^2(A) = \| w_* q^\top \|^2_F = \| w_* \|^2$.
\end{proof}

\begin{lemma} \label{lem:smooth}
(smoothness)
Assume that $\dist(W_{t_0}, w_*) := \| W_{t_0} - w_* q_{t_0}^\top \|_F \leq \nu \| w_* \|_2$ where $W_{t_0} \in \reals^{d\times K}$, $\nu > 0$, $q_{t_0} := \arg\min_{q \in \reals^K: \| q\|_2 \leq 1}  \|  W - w_*q^\top \|_F$, and $w_* \in \reals^d$ being the teacher neuron. Then,
$
\textstyle
\| \nabla^2 F(W) \|_2
\leq  (15 + 16 \nu^2 ) \| w_* \|^2.
$
\end{lemma}

\begin{proof}
For notation brevity, we will suppress script $t_0$ in the following (i.e. $W \leftarrow W_{t_0}$, $q \leftarrow q_{t_0}$). We have that
\begin{equation}
\begin{split}
\| \nabla^2 F(W) \|_2
& = \| \E \big[ [\big( \| x^\top W \|^2_2 - \| x^\top w_* q^\top \|^2_2  \big) I_K + 2 W^\top x x^\top W ]    \otimes x x^\top \big]\|_2
\\ & \leq \| \E \big[ [ \big| \| x^\top W \|^2_2 - \| x^\top w_* q^\top \|^2_2 \big|  I_K + 2 \| x^\top W \|^2_2 I_K ]    \otimes x x^\top \big]\|_2
\\ & \overset{(a)}{ \leq} 
\| \E \big[ \big| \|x^ \top W \|^2_2 - \| x^\top w_* q^\top \|^2_2  \big|  x x^\top  \big]\|_2
+ 
2 \| \E \big[ \| x^\top W \|^2_2 x x^\top \big] \|_2
\\ & \overset{(b)}{ = } 
\| \E \big[ ( \|x^ \top W \|^2_2 - \| x^\top w_* q^\top \|^2_2 )  x x^\top  \big]\|_2
+ 2 \| \| W \|^2_F I_d + 2 W W^\top \|_2
\\ & \overset{(c)}{ = } 
\| ( \| W \|^2_F - \| w_* q^\top \|^2_F ) I_d + 2 ( W W^\top - w_* w_*^\top)  \|_2
+ 2 \| \| W \|^2_F I_d + 2 W W^\top \|_2
\\ & \leq
 7 \| W \|^2_F + \| w_* q^\top \|^2_F  + 2 \| W W^\top - w_* w_*^\top \|_2 \\ & \leq
 (15 + 16 \nu^2 ) \| w_* \|^2 , 
\end{split}
\end{equation}
where (a) is due to $\| I \otimes A \|_2 \leq \| I \|_2 \| A \|_2 = \| A \|_2$, (b,c) is due to Lemma~\ref{lem:gaussian},
and the last inequality is by setting $W = w_* q^\top + \theta H$ with a $H$ satisfying $\| H \|_F =1 $ and $\theta := \nu \| w_* \|$.

\end{proof}

Lemma~\ref{lem:strcvx} and Lemma~\ref{lem:smooth} together implies that when the the iterate is in a neighborhood of a global optimal solution, then gradient descent has a linear convergence rate.

\noindent
\textbf{Lemma~\ref{lem:linear}:}
\textit{
(locally linear convergence)
Suppose that at time $t_0$, $\dist(W_{t_0}, w_*) :=\| W_{t_0} - w_* q_{t_0}^\top \| \leq \nu \| w_* \|$ where $W_{t_0} \in \reals^{d\times K}$, $\nu > 0$ satisfies $2 - 14 \nu - 2\nu^2 > 0$, and $q_{t_0} := \arg\min_{q \in \reals^K: \| q\|_2 \leq 1}  \|  W_{t_0} - w_*q^\top \|$. Then,
gradient descent with the step size $\eta \leq \frac{ 2 - 14 \nu - 2\nu^2 }{  (15 + 16 \nu^2 )^2 \| w_* \|^4}$
generates iterates $\{ W_{t} \}_{t \geq t_0}$ satisfying
\[
\textstyle
\dist^2(W_{t+1},w_*) \leq 
( 1 -  \eta  (2 - 14 \nu -2 \nu^2) ) \dist^2( W_{t}, w_*).
\]
}

\begin{proof}

We have that
\begin{equation} \label{eq:b0}
\begin{split}
\dist^2( W_{t+1}) & := \| W_{t+1} - w_* q_{t+1}^\top \|^2_F 
\\ & \leq  \| W_{t+1} - w_* q_t^\top \|^2_F
\\ & =  \| W_t - \eta \nabla F(W_t) - w_* q_t^\top \|^2_F
\\ & \overset{(a)}{=}  \| w_t - \bar{w}_*  - \eta \vec\big( \nabla F(W_t) - \nabla F(w_* q_t^\top) \big) \|^2_F
\\ & \overset{(b)}{=}  \| w_t - \bar{w}_*  - \eta \big( \int_0^1 \nabla^2 F(W_t(\tau) ) d \tau \big) (w_t - \bar{w}_*) \|^2_F
\\ & = ( w_t - \bar{w}_* )^\top \big( I_{dK} -  \eta \int_0^1 \nabla^2 F(W_t(\tau) ) d \tau      \big)^2 ( w_t - \bar{w}_* )  
\\ & \leq \|  w_t - \bar{w}_* \|^2_2 - 2 \eta 
 ( w_t - \bar{w}_* )^\top \big( \int_0^1 \nabla^2 F(W_t(\tau) ) d \tau      \big) ( w_t - \bar{w}_* )
\\ &  + \eta^2 \|    \int_0^1 \nabla^2 F(W_t(\tau) ) d \tau \|_2^2 \|  w_t - \bar{w}_* \|^2_2. 
\end{split}
\end{equation}
where $(a)$ we use the notations that 
$w_t := \vec(W_t)$ and $\bar{w}_* := \vec(w_* q_t^\top)$ 
and that $\nabla F (w_* q_t^\top) = 0$
and $(b)$ we denote $W_t(\tau) := w_* q_t^\top + \tau \big( W_t - w_* q_t^\top \big)$.

Notice that
$\dist( W_t(\tau), w_*) \leq \| W_t(\tau) - w_* q_t^\top \| 
= \tau \| W_t - w_* q_t^\top \| \leq \nu \| w_* \|$.
So we can invoke Lemma~\ref{lem:strcvx} to obtain that
\begin{equation} \label{eq:b1}
\begin{split}
& ( w_t - \bar{w}_* )^\top \big( \int_0^1 \nabla^2 F(W_t(\tau) ) d \tau      \big) ( w_t - \bar{w}_* )
\\ & 
\geq 2 \tr( q_t w_*^\top ( W_t - w_* q_t^\top ) q_t w_*^\top ( W_t - w_* q_t^\top ) )
+ (2 - 14 \nu -2 \nu^2)  \| w_* \|^2 \| W_t - w_* q_t^\top \|^2_F
\\ & 
\geq  (2 - 14 \nu -2 \nu^2)  \| w_* \|^2 \| W_t - w_* q_t^\top \|^2_F.
\end{split}
\end{equation}
where the last inequality is due to that
$q_t w_*^\top ( W_t - w_* q_t^\top )$ is a symmetric matrix
since $q_t= \frac{W_t^\top w_*}{ \| W_t^\top w_* \| }$
so that 
$\tr( q_t w_*^\top ( W_t - w_* q_t^\top ) q_t w_*^\top ( W_t - w_* q_t^\top ) ) = \| q_t w_*^\top ( W_t - w_* q_t^\top ) \|^2_F \geq 0$.
Furthermore, by Lemma~\ref{lem:smooth}, we have that
\begin{equation}
\| \nabla^2 F(W) \|_2
\leq  (15 + 16 \nu^2 ) \| w_* \|^2.
\end{equation}
So
\begin{equation} \label{eq:b2}
\eta^2 \|    \int_0^1 \nabla^2 F(W_t(\tau) ) d \tau \|_2^2 \|  w_t - \bar{w}_* \|^2_2 
\leq \eta^2 (15 + 16 \nu^2 )^2 \| w_* \|^4 \|  w_t - \bar{w}_* \|^2_2 .
\end{equation}
Combining (\ref{eq:b0}), (\ref{eq:b1}), and (\ref{eq:b2}), we get
\begin{equation}
\begin{split}
\dist^2( W_{t+1}, w_*) 
 \leq & \|  w_t - \bar{w}_* \|^2_2 - 2 \eta 
 ( w_t - \bar{w}_* )^\top \big( \int_0^1 \nabla^2 F(W_t(\tau) ) d \tau      \big) ( w_t - \bar{w}_* )
\\ &  + \eta^2 \|    \int_0^1 \nabla^2 F(W_t(\tau) ) d \tau \|_2^2 \|  w_t - \bar{w}_* \|^2_2
\\ \leq & \|  W_t - w_* q_t^\top \|^2_F - 2 \eta  (2 - 14 \nu - 2 \nu^2)  \| w_* \|^2 \| W_t - w_* q_t^\top \|^2_F
\\ & + \eta^2 (15 + 16 \nu^2 )^2 \| w_* \|^4 \| W_t - w_* q_t^\top \|^2_F
\\ \leq & ( 1 -  \eta  (2 - 14 \nu - 2 \nu^2) ) \dist^2( W_{t}, w_*), 
\end{split}
\end{equation}
where the inequality we use that $\eta \leq \frac{ 2 - 14 \nu - 2 \nu^2 }{
  (15 + 16 \nu^2 )^2 \| w_* \|^4} $.

\end{proof}

\begin{lemma} (Lemma~12 in \cite{LMCC19}) \label{lem:gaussian}
Suppose $x \sim N(0,I_d)$. Then for any fixed matrices $W,V \in \reals^{d \times r}$. We have that
\[
\begin{aligned}
\E[ \| x^\top V \|_2^2 \| x^\top W \|_2^2 ] & = \| V \|^2_F \| W \|^2_F +  2 \| V^\top W \|^2_F
\\
\E[(x^\top W V^\top x)^2 ] & = \big( \tr(W^\top V) \big)^2 + \tr( W^\top V W^\top V ) + \| W V^\top \|^2_F.
\end{aligned}
\]
\end{lemma}

\section{Proof of Theorem~\ref{lem:single}} \label{app:single}

\begin{proof}
Since it is clear about the number of the student neuron (which is $1$), in the following, we suppress the subscript $\#1$ and the superscript $(1)$ for the brevity of notations.
Recall the dynamics,
\begin{equation}
\begin{aligned}
w_{t+1}^{\parallel} 
& = 
w_t^{\parallel} \big( 1 + \eta ( 3 \| w_* \|^2 - 3 \| w_t \|^2   ) \big) 
\\ 
w_{t+1}^{\perp} 
& = 
w_t^{\perp} \big( 1 + \eta ( \| w_* \|^2 - 3 \| w_t \|^2   ) \big).
\end{aligned}
\end{equation}
and note that 
$w_t^{\parallel} = w_t[1] \| w_* \|$ so we have that
\begin{equation}
w_{t+1}[1] 
= 
w_t[1] \big( 1 + \eta ( 3 \| w_* \|^2 - 3 \| w_t \|^2   ) \big) 
\end{equation}

For a number $\gamma > 0$,
define
$\textstyle
T_{\gamma}:= \min \{t : 
| |w_t[1]| - \| w_* \| | \leq \gamma \text{ and } \| w_t^\perp \| \leq \gamma \}$.
We can decompose the square of the distance term as follows,
\[
\dist^2( W_t^{\#1}, w_*) 
 = | |w_{t}[1]| - \| w_* \| |^2 +  \| w_t^\perp \|^2.
\]
In the latter part of this proof, we will show that $\| w_t^{\perp} \| \leq \gamma$ for all $t \leq T_{\gamma}$ .
Let us upper-bound the norm of $\| w_t \|^2$ for $t \leq T_{\gamma}$ as follows.
\begin{equation}
\begin{split}
\| w_t \|^2 & =  w_t[1]^2 + \| w_t^\perp \|^2
\leq ( \| w_* \| - \gamma)^2  + \gamma^2
\\ & 
=  \| w_* \|^2 - 2 \gamma \| w_* \| +  2 \gamma^2,
\end{split}
\end{equation}
where the inequality is because $w_t[1]^2 \leq ( \| w_* \| - \gamma )^2$ for $t \leq T_{\gamma}$ by the definition and that $\| w_t^{\perp} \| \leq \gamma$ for all $t \leq T_{\gamma}$ proved in the latter part.
Hence, we have that
\begin{equation} \label{eq:sgrow}
\begin{split}
| w_{t+1}[1] | & = | w_{t}[1] | \big( 1 + \eta ( 3 \| w_* \|^2 - 3 \| w_t \|^2   ) \big)
\geq | w_t[1] | \big( 1 + \eta ( 3 \|w_*\|^2 - 3 ( \| w_* \|^2 - 2 \gamma \| w_* \| +  2 \gamma^2 ) )    \big)
\\ & := | w_t[1] | \big( 1 + \eta \Delta  \big)
= | w_0[1] | \big( 1 + \eta \Delta  \big)^{t+1},
\end{split}
\end{equation}
with
\[
\Delta := 6 \gamma ( \| w_* \| - \gamma ) > 0.
\]
Based on (\ref{eq:sgrow}), it takes at most number of iterations
\[
\begin{aligned}
T_{\gamma} \leq \frac{ \log ( \frac{ \| w_* \| - \gamma  }{ |w_0[1]| }  ) }{ \log( 1 + \eta \Delta) }
\end{aligned}
\]
for $w_t[1]$ to rises above $\| w_* \| - \gamma$ if $w_0[1] > 0$.
Similarly, if $w_0[1] < 0$, it takes $T_{\gamma}$ iterations for $w_0[1]$ to satisfy $w_t[1] \leq - \| w_* \| + \gamma$.

Now we switch to show that for $0 \leq t \leq T_{\gamma}$,
$\gamma \geq \| w_{t}^\perp \|. $
We are going show that the perpendicular component $\| w_t^\perp \|$ start decaying before it could have increased above $\gamma$. 
From the dynamics 
$\| w_{t+1}^{\perp} \| = \| w_t^{\perp} \| \big( 1 + \eta ( \| w_* \|^2 - 3 \| w_t \|^2   ) \big)$, we see that once $\| w_t \|^2 \geq \frac{1}{3} \| w_* \|^2$, the size of the perpendicular component $\| w_t^{\perp} \|$ starts decaying. On the other hand, $\| w_t^{\perp} \|$ is increasing when $\| w_t \|^2 \leq \frac{1}{3} \| w_* \|^2$,
which leads to the following before the perpendicular component starts decaying,
\begin{equation} \label{eq:ngrow}
\begin{split}
\| w_t \|^2 \geq w_t^2[1] = \frac{1}{ \| w_* \|^2} 
| w_t^\parallel |^2 
& = \frac{1}{ \| w_* \|^2} | w_{t-1}^\parallel |^2 \big( 1 + \eta ( 3 \| w_* \|^2 -3 \| w_{t-1} \|^2 )  \big)^2
\\ & 
\geq \frac{1}{ \| w_* \|^2} | w_{t-1}^\parallel |^2 \big( 1 + 2 \eta  \| w_* \|^2  \big)^2
\\ &
\geq \frac{1}{ \| w_* \|^2} | w_{0}^\parallel |^2 \big( 1 + 2 \eta  \| w_* \|^2  \big)^{2t} =
| w_{0}[1] |^2 \big( 1 + 2 \eta  \| w_* \|^2  \big)^{2t}, 
\end{split}
\end{equation}
which means that that the size of $\| w_t \|^2$ grows at the rate at least
$\big( 1 + 2 \eta  \| w_* \|^2  \big)^{2}$ before the perpendicular component $\| w_t^\perp \|$ starts decaying.

The inequality (\ref{eq:ngrow}) also implies that the number of iterations such that 
$\| w_t \|^2 \leq \frac{1}{3} \| w_* \|^2$ is 
 at most 
\begin{equation}
t^* \leq \frac{1}{2} \frac{ \log ( \frac{\| w_* \|^2 }{3 | w_0[1] |^2 }  ) }{ \log \big( (1 + 2 \eta  \| w_* \|^2)^2  \big)  }.
\end{equation}
After $t^*$, we have that $\| w_t^\perp \|$ is decaying.
So we only have to show that for $0 \leq t \leq t^*$
 the perpendicular component never rise above $\gamma$.
It suffices to show that an upper bound of $\| w_{t}^\perp \|$ for $0 \leq t \leq t^*$
is not greater than $\gamma$,
\begin{equation}
\begin{split}
\| w_{t}^{\perp} \| 
& = 
\| w_{t-1}^{\perp} \| \big( 1 + \eta ( \| w_* \|^2 - 3 \| w_{t-1} \|^2   ) \big) 
\\ &
\leq 
\| w_{t-1}^{\perp} \| \big( 1 + \eta ( \| w_* \|^2 - 3 w_{t-1}^2[1] ) \big)
\\ & 
\leq 
\| w_0^{\perp} \| \cdot
\Pi_{s=0}^{t^*-1} \big( 1 + \eta ( \| w_* \|^2 - 3 w_s^2[1] ) \big)
\\ & 
\overset{(a)}{\leq} 
\| w_0^{\perp} \| \cdot
\Pi_{s=0}^{t^*-1} \big( 1 + \eta ( \| w_* \|^2 - 3 w_0^2[1] (1 + \eta \Delta)^{2s}     )    \big)
\\ & 
\overset{?}{\leq} \gamma,
\end{split}
\end{equation}
where (a) we use (\ref{eq:sgrow}). 
By taking logarithm on the both sides of $\overset{?}{\leq}$ , it suffices to show that
\begin{equation}
\begin{split}
& \log \| w_0^\perp \| + \sum_{s=0}^{t^* - 1} 
\log \big(  1 + \eta ( \| w_* \|^2 - 3 w_0^2[1] (1 + \eta \Delta)^{2s}     ) \big)
\overset{?}{\leq} \log \gamma.
\end{split}
\end{equation}

Using the fact that $\log x \geq 1 - \frac{1}{x}$ and that $\log (1+x) \leq x$ for $x>-1$,
it suffices to show that
\[
\begin{split}
& \sum_{s=0}^{t^* - 1} 
\eta ( \| w_* \|^2 - 3 w_0^2[1] (1 + \eta \Delta)^{2s}     ) \overset{?}{ \leq }
1 - \frac{ \| w_0^\perp \| }{ \gamma }.
\end{split}
\]
which can be guaranteed if 
\begin{equation} \label{eq:tif}
\begin{split}
& \eta t^* \| w_* \|^2 \leq  3 \eta w_0^2[1]
\frac{  (1+\eta \Delta)^{2t^*} -1  }{ (1+\eta \Delta)^2 -1  }
+ \frac{9}{10},
\end{split}
\end{equation}
where we use that $\gamma \geq 10 \| w_0^\perp\|$. 
So it suffices to have the step size satisfy $\eta = c / \| w_* \|^2$ for some sufficiently small constant $c>0$. 

\end{proof}

\section{Proof of Lemma~\ref{lem:syn_up} and Lemma~\ref{lem:norm_up}} \label{app:oparam}

\noindent
\textbf{Lemma~\ref{lem:syn_up}:}
\textit{ 
Suppose that the approximated dynamics (\ref{eq:multi-dyn2}) and (\ref{multi-par}) hold from iteration $0$ to iteration $t$.
Then, the network with a single neuron and an over-parametrized network trained by GD with the same step size $\eta$
has
$\sqrt{\sum_{k=1}^K | \tparw{k}{K}{t} |^2} \gtrsim \sqrt{ ( 1 - 2 t \theta) } \sqrt{K} | \tparw{1}{1}{t} |$.
}

\begin{proof} 

From the dynamics (\ref{eq:multi-dyn2}), we have that
\begin{equation} \label{eq:par1}
\begin{aligned}
| \tparw{k}{K}{t} |
\geq  &
( 1 - \theta)
| \tparw{k}{K}{t-1} |  \big( 1 + \eta ( 3 \| w_* \|^2 - 3 \| \tw{k}{K}{t-1} \|^2  ) \big)
\\ \geq  &
( 1 - \theta )^t | \tparw{k}{K}{0} |
\cdot \Pi_{s=0}^{t-1}  
\big( 1 + \eta ( 3 \| w_* \|^2 - 3 \| \tw{k}{K}{s} \|^2  ) \big).
\end{aligned}
\end{equation}
Therefore,
\begin{equation} \label{eq:par2}
\begin{split}
&\frac{ \sum_{k=1}^K | \tparw{k}{K}{t} |^2  }{ | \tparw{1}{1}{t} |^2 }
\overset{(a)}{ \geq } 
\frac{ \sum_{k=1}^K 
( 1 - \theta )^{2t}  |\tparw{k}{K}{0}|^2
\cdot \Pi_{s=0}^{t-1}  
\big( 1 + \eta ( 3 \| w_* \|^2 - 3 \| \tw{k}{K}{s} \|^2  ) \big)^2.
   }{  |\tparw{1}{1}{0}|^2
\cdot \Pi_{s=0}^{t-1}  
\big( 1 + \eta ( 3 \| w_* \|^2 - 3 \| \tw{1}{1}{s} \|^2  ) \big)^2    }
\\ &
\overset{(b)}{\gtrsim}
\sum_{k=1}^K 
( 1 - \theta )^{2t} \frac{ |\tparw{k}{K}{0}|^2 }{  |\tparw{1}{1}{0}|^2     }
\overset{(c)}{\geq} ( 1 - 2 t \theta ) \sum_{k=1}^K \frac{ |\tparw{k}{K}{0}|^2 }{  |\tparw{1}{1}{0}|^2     }
\overset{(d)}{\gtrsim} ( 1 - 2 t \theta ) K,
\end{split}
\end{equation}
where (a) is due to (\ref{eq:par1}) and the recursive expansion, (b) is by using that $\| \tw{1}{1}{s} \|^2 \gtrsim \| \tw{k}{K}{s} \|^2$ 
so that 
$\Pi_{s=0}^{t-1} \big( 1 + \eta ( 3 \| w_* \|^2 - 3 \| \tw{k}{K}{s} \|^2  ) \big)^2 \gtrsim \Pi_{s=0}^{t-1} \big( 1 + \eta ( 3 \| w_* \|^2 - 3 \| \tw{1}{1}{s} \|^2  ) \big)^2 $,
(c) is by $(1+x)^r \geq 1 + rx $ for $x>-1$ and $r \in R \setminus (0,1)$, and (d) is by the initialization such that each node is i.i.d. initialized from the gaussian distribution.

\end{proof}

\noindent
\textbf{Lemma \ref{lem:norm_up}:} 
\textit{ 
Suppose that $\eta \leq \frac{1}{3 \| w_* \|^2}$.
By following the conditions as Lemma~\ref{lem:syn_up},
we have that
$
\| \tpenw{k}{K}{t} \|
\lesssim \frac{ |\tparw{k}{K}{t}| \| \tpenw{k}{K}{0} \|  }{ | \tparw{k}{K}{0} |   } \frac{1}{\psi^t}
\lesssim 
\frac{ | \tparw{1}{1}{t} | \| \tpenw{k}{K}{0} \|  }{ | \tparw{k}{K}{0} |   } \frac{1}{\psi^t},$
where $\psi:= (1- \theta - \vartheta - \theta \vartheta  ) \big( 1 + \eta \| w_* \|^2  \big)$.
}

\begin{proof}
For brevity, let us suppress the subscript $\#K$ for the moment.
By the approximated dynamics (\ref{eq:multi-dyn2}), we have that 
\begin{equation}
\begin{split}
& \frac{ |w_{t+1}^{(k),\parallel}| }{ \| w_{t+1}^{(k),\perp} \| }
 \geq
\frac{ (1-\theta) ( 1 + \eta ( 3 \| w_* \|^2 - \| w_t^{(k)}\|^2  ) )  }{ (1+\vartheta)(   1 + \eta ( \| w_* \|^2  - \| w_t^{(k)}\|^2  ) ) }
\frac{ |w_{t}^{(k),\parallel}| }{ \| w_{t}^{(k),\perp} \| }
\\ & \geq (1- \theta - \vartheta - \theta \vartheta ) \big( 1 + 2 \eta \| w_* \|^2 - \eta ^2 ( 3 \| w_* \|^2 - \| w_t^{(k)}\|^2  )  ( \| w_* \|^2  - \| w_t^{(k)}\|^2  )   \big)  \frac{ |w_{t}^{(k),\parallel}| }{ \| w_{t}^{(k),\perp} \| }
\\ & \gtrsim \psi \frac{ |w_{t}^{(k),\parallel}| }{ \| w_{t}^{(k),\perp} \| },
\end{split}
\end{equation}
where the second inequality we use $\frac{1+a}{1+b} \geq 1 + a - b - ab$ for $b \geq -1$ and
 the last inequality is because 
$(1- \theta - \vartheta - \theta \vartheta ) \big( 1 + 2 \eta \| w_* \|^2 - \eta ^2 ( 3 \| w_* \|^2 - \| w_t^{(k)}\|^2  )  ( \| w_* \|^2  - \| w_t^{(k)}\|^2  ) \big)
\gtrsim
(1- \theta - \vartheta - \theta \vartheta  ) \big( 1 + \eta \| w_* \|^2  \big)
:= \psi,$
as $\| \tw{k}{K}{t}\|^2 \lesssim \| \tw{1}{1}{t} \|^2 < \| w_* \|^2$ and that $\eta \leq \frac{1}{3 \| w_* \|^2}$.

Consequently,
$\frac{ |w_{t}^{(k),\parallel}| }{ \| w_{t}^{(k),\perp} \| }
\gtrsim 
\psi^t
\frac{ |w_{0}^{(k),\parallel}| }{ \| w_{0}^{(k),\perp} \| }.$
Namely,
\begin{equation}
\begin{split}
\| \tpenw{k}{K}{t} \| \lesssim
\frac{ |\tparw{k}{K}{t}| \| \tpenw{k}{K}{0} \|  }{ | \tparw{k}{K}{0} |   } \frac{1}{\psi^t}.
\end{split}
\end{equation}
Moreover, by the condition that
$|\tparw{1}{1}{t}| \gtrsim |\tparw{k}{K}{t}|,$
we have that
\begin{equation}
\begin{split}
\| \tpenw{k}{K}{t} \|
& \lesssim \frac{ |\tparw{k}{K}{t}| \| \tpenw{k}{K}{0} \|  }{ | \tparw{k}{K}{0} |   } \frac{1}{\psi^t}
\lesssim
\frac{ |\tparw{1}{1}{t}| \| \tpenw{k}{K}{0} \|  }{ | \tparw{k}{K}{0} |   } \frac{1}{\psi^t}.
\end{split}
\end{equation}

\end{proof}


\end{document}